\documentclass{article}
\usepackage{iclr2026_conference,times}

\usepackage{hyperref}
\usepackage{url}
\usepackage{graphicx}
\usepackage{xcolor}
\usepackage{booktabs}
\usepackage{tabularx}
\usepackage{multirow}
\usepackage{amsmath,amssymb,amsthm}
\newtheorem{proposition}{Proposition}
\newtheorem{remark}{Remark}

\usepackage[ruled,vlined,linesnumbered]{algorithm2e}
\SetAlgoNlRelativeSize{-1}
\SetAlgoCaptionSeparator{:}
\usepackage{tikz}
\usepackage{pgfplots}
\pgfplotsset{compat=1.17}
\usepgfplotslibrary{groupplots}
\usetikzlibrary{arrows.meta,positioning,shapes.geometric,calc,fit,backgrounds,decorations.pathreplacing,patterns}
\usepackage{microtype}
\usepackage{enumitem}
\usepackage{listings}
\hypersetup{
  colorlinks=true,
  linkcolor=blue!50!black,
  citecolor=blue!50!black,
  urlcolor=blue!60!black
}

\newcommand{\sys}{\textsc{SuperBrowser}\xspace}

\renewcommand{\sys}{\textsc{SuperBrowser}}
\newcommand{\Vn}{$[V_n]$}
\newcommand{\Vidx}[1]{$[V_{#1}]$}

\newcommand{\code}[1]{\texttt{\small #1}}

\iclrfinalcopy

\title{RunAgent \sys{}: A Theory of Autonomous Web\\Navigation Grounded in Human Browsing Behaviour}

\author{Radeen Mostafa \\
RunAgent AI
\And
Sawradip Saha \\
RunAgent AI
}

\begin{document}
\maketitle
\lhead{Preprint. Work in progress.}
\rhead{\today}
\cfoot{\thepage}
\thispagestyle{fancy}

\begin{abstract}
We present \sys, an autonomous web-navigation agent designed against a single
guiding hypothesis: \emph{a web agent should browse the way a person browses}.
A human reading a page does not retain every pixel they have seen; they look at
a few candidate targets, decide on one, and remember only what is needed to keep
the goal alive. We operationalize this perception--cognition--action triad as
three coupled mechanisms. First, a vision-first bounding-box pipeline labels
candidate interactive regions on every screenshot and feeds them, asynchronously
prefetched, to the language model so that the ``eye'' precedes the ``hand''.
Second, a three-role brain---an \emph{Orchestrator} that classifies and routes,
a \emph{Planner} that evaluates progress every few steps, and a \emph{Worker}
that emits per-step actions---separates strategic from operational reasoning.
Third, a structured \emph{Ledger} stores only what a person would: the goal,
the last three actions, a small set of facts and dead-ends, and a handful of
checkpoints; a six-phase eviction loop systematically discards stale screenshots,
state blobs, and reasoning traces from the live context. Action execution is a
three-tier click cascade (Chrome DevTools Protocol $\to$ Puppeteer $\to$
scripted) with humanized B\'ezier motion, plus a chevron-aware bounding-box
snapper that resolves the ``small arrow beside a large label'' ambiguity. On the
Mind2Web Hard benchmark (66 tasks), \sys{} attains \textbf{89.47\%} success,
placing third overall and ahead of every published open / research browser-agent
baseline by a large margin. We argue that the gain comes not from any single
trick but from the consistent application of a cognitive contract throughout the
system.
\end{abstract}

\begin{center}
Code: \url{https://github.com/runagent-dev/runagent-superbrowser}
\end{center}

\section{Introduction}
\label{sec:intro}

A person booking a hotel does not photograph every page they visit. They glance
at the page, narrow their attention to a few candidate buttons, click one, and
on the next page they remember almost nothing of the previous one except the
goal (``three nights in Khulna under \$40''), the last action (``I picked the
April 23 check-in''), and perhaps one or two salient facts (``the cheapest
result was \$32''). They certainly do not carry forward the DOM tree, the list
of every link they saw, or the failure messages from three pages ago. This
parsimony is not a bug of human cognition; it is the architecture that makes
long-horizon goal pursuit feasible at all under bounded working memory
\citep{miller1956magical,sweller1988cognitive}.

Most current LLM-based web agents do the opposite. They accumulate
observations, screenshots, element lists, and reasoning traces into a single
ever-growing prompt, then ask the same model to keep deciding. The result is
predictable: as steps accumulate, the prompt-cache hit rate collapses, hidden
tokens drift attention away from the goal, and reliability on long tasks
degrades sharply \citep{deng2023mind2web,zhou2024webarena}. The community has
responded with longer context windows, but \emph{capacity} is not the same as
\emph{discipline}: a 200K-token context populated by thirty screenshots is
still a worse decision-making environment than a 12K-token context populated by
a goal, three recent actions, and a handful of facts.

In this paper we make the cognitive analogy load-bearing rather than rhetorical.
We treat web navigation as an instance of the classical
perception--cognition--action triad and design an agent, \sys{}, in which every
component honours the analogy:

\begin{itemize}[leftmargin=*,itemsep=2pt,topsep=2pt]
  \item \textbf{Perception} is a vision-bounding-box pipeline: a vision model
        labels candidate interactive regions $[V_0],[V_1],\dots$ on each
        screenshot, with DOM enrichment (\code{aria-expanded}, \code{is\_active},
        \code{group\_label}) and asynchronous prefetch.
  \item \textbf{Cognition} is split across three roles---\emph{Orchestrator}
        (verb-classifier routing between search and browse), \emph{Planner}
        (every-$N$-step progress evaluation), \emph{Worker} (per-step action
        emission)---each with a sharply scoped prompt.
  \item \textbf{Memory} is a structured \emph{Ledger}: a goal, a plan, a
        bounded recency deque of three steps, a fact dictionary, a list of
        dead-ends, and a list of checkpoints. A six-phase eviction loop runs
        every iteration and discards screenshots, failures, state blobs, and
        reasoning blocks that have aged out.
  \item \textbf{Action} is a three-tier click cascade (CDP mouse with
        humanized B\'ezier curves $\to$ Puppeteer $\to$ scripted) and a
        chevron-aware bounding-box snapper that resolves common UI ambiguities
        (\emph{e.g.}, a tiny ``$\blacktriangledown$'' arrow next to a wide
        ``United States'' label).
\end{itemize}

\paragraph{Contributions.} (i) We articulate a cognitive theory of autonomous
web navigation as a system contract, not as a metaphor, and identify three
falsifiable predictions it makes (\S\ref{sec:cogtheory}). (ii) We propose a
three-role brain architecture with verb-classifier routing, vision prefetch,
and structured episodic memory (\S\ref{sec:arch}, \S\ref{sec:memory}).
(iii) We introduce a six-phase context-eviction loop that holds steady-state
context roughly constant per iteration while preserving a structured ledger
(\S\ref{sec:memory}). (iv) We describe and evaluate a chevron-aware
bounding-box snapping algorithm and a three-tier humanized click cascade that
resolve longstanding failure modes of vision-grounded web agents
(\S\ref{sec:action}). (v) On Mind2Web Hard (66 tasks), \sys{} attains
\textbf{89.47\%}, placing third overall and ahead of every published open /
research browser-agent baseline by an order of magnitude
(\S\ref{sec:eval}).

\section{Related Work}
\label{sec:related}

\paragraph{Web-navigation benchmarks.}
Mind2Web \citep{deng2023mind2web} introduced a large-scale crawl-and-replay
corpus of human web traces and a Hard subset of $66$ long-horizon tasks that
has since become a standard yardstick for browser-agent generality.
Online-Mind2Web \citep{xue2025illusion} re-evaluates $300$ live tasks on
$136$ high-traffic sites and reports a sobering ``illusion of progress''
result: most commercial agents underperform a $2024$ academic baseline, and
even top systems score below $50\%$ on the hard split. Mind2Web~2
\citep{gou2025mind2web2} adds $130$ new long-horizon tasks evaluated with an
agent-as-a-judge protocol. Earlier sandboxed environments include WebArena
\citep{zhou2024webarena} and its multimodal extension VisualWebArena
\citep{koh2024visualwebarena}; the simulated end of the spectrum is covered
by MiniWoB \citep{shi2017wob,liu2018miniwob}. Operating-system-level
analogues include OSWorld \citep{xie2024osworld}; the BrowserGym ecosystem
\citep{drouin2024browsergym} provides a unified harness for web-agent
research, and WebCanvas \citep{pan2024webcanvas} and WebGames
\citep{lu2025webgames} test online and challenging-task conditions
respectively.

\paragraph{Visual grounding for LLM agents.}
Set-of-Mark prompting \citep{yang2023setofmark} demonstrated that overlaying
numbered marks on candidate UI regions sharply improves a multimodal LLM's
ability to act on a page. SeeAct \citep{zheng2024seeact} formalised
``planning-then-grounding'' for GPT-4V on web tasks; CogAgent
\citep{hong2024cogagent} pushes the vision encoder to $1120{\times}1120$
pixels to read small UI text on screenshots; UI-TARS \citep{qin2025uitars}
packages a dual-resolution VLM and a native action interface into an open
GUI-agent stack.

\paragraph{Browser-agent frameworks.}
ReAct \citep{yao2023react} established the interleaved think--act trace that
underpins most modern agents; ReWOO \citep{xu2023rewoo} decouples reasoning
from observation to reduce token cost; Toolformer
\citep{schick2023toolformer} taught language models to invoke external
tools from self-supervised data. Browser-specific predecessors include
WebGPT \citep{nakano2021webgpt}, WebVoyager \citep{he2024webvoyager},
AutoWebGLM \citep{lai2024autowebglm}, and the open-source
\texttt{browser-use} library \citep{browseruse2024}. Recent work surveys
architectural and security considerations across the resulting agent
landscape \citep{shahbandeh2025browseragents,liang2024security}. Closely
related to our role-separated design is the dispatcher / supervisor /
executor split used in robotics \citep{ahn2022saycan} and
operating-system-style agents \citep{wu2024oscar}.

\paragraph{Memory-aware agents.}
Reflexion \citep{shinn2023reflexion} accumulates verbal self-critiques
across episodes; Voyager \citep{wang2024voyager} grows a library of
reusable skills; Generative Agents \citep{park2023generative} maintain
reflective memory streams. MemGPT \citep{packer2023memgpt} treats the
context window as ``main memory'' over a paged archive, drawing an
explicit analogue to operating systems; A-Mem \citep{xu2025amem}
structures agentic memory as linked notes retrievable by relevance;
LongMem \citep{wang2023longmem} augments language models with an external
retriever-trained memory tier. Recent surveys catalogue this rapidly
expanding landscape \citep{zhang2025memorysurvey}.

\medskip
\noindent
Against this body of work, \sys{} positions itself as the first system
to combine cognitive-theory-motivated bounded context, role-sliced
strategic/tactical/operational reasoning, and a procedural-DOM tier that
\emph{subtracts} stale perception on every turn rather than accumulating
it.

\section{A Cognitive Theory of Web Navigation}
\label{sec:cogtheory}

We start from a deliberately strong claim: the best decomposition of an
autonomous web agent is the same decomposition that cognitive science has
arrived at for human information-seeking behaviour. In this section we
articulate that decomposition (\S\ref{sec:triad}), ground it in working-memory
and episodic-recall theory (\S\ref{sec:memcog}), connect it to empirical
findings from four decades of eye-tracking and information-foraging studies on
real human web browsing (\S\ref{sec:humanbrowse}), and extract three
falsifiable predictions our experiments later test (\S\ref{sec:predictions}).

\subsection{The Perception--Cognition--Action Triad}
\label{sec:triad}

The classical sensorimotor pipeline has three stages: a percept is constructed
from raw signal, a decision is taken in working memory under top-down task
constraints, and an action is dispatched to effectors. Figure~\ref{fig:cogstack}
shows our mapping. A web agent's ``retina'' is a screenshot, useless unless
segmented into candidate interactive regions; we therefore treat the vision
model as a \emph{candidate generator}, the same way the visual system generates
proto-objects before attention selects one---a strategy formalised in the
saliency-map literature \citep{itti1998saliency,itti2001computational}. The
agent's ``brain'' chooses which \Vn{} to act on, separated into a slow
strategic loop (Planner) and a fast operational loop (Worker), echoing the
System~2 / System~1 distinction \citep{kahneman2011thinking}. The agent's
``hand'' is the click cascade, which must be both reliable and---against bot
detection---plausibly human in its kinematics.

\begin{figure}[t]
\centering
\includegraphics[width=\linewidth]{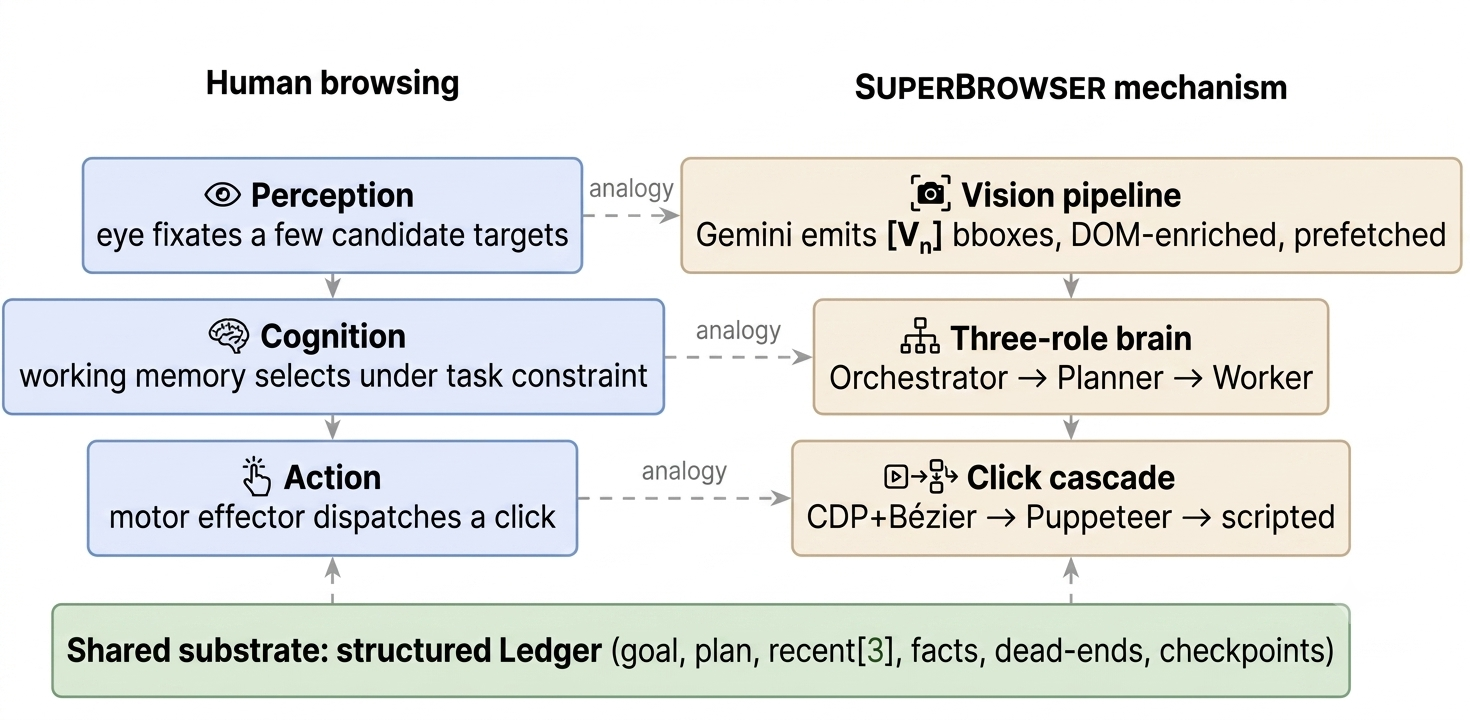}
\caption{The perception--cognition--action triad as instantiated in \sys.
The human column lists the cognitive primitives; the system column lists the
corresponding mechanisms. Memory is a shared substrate that constrains both
cognition (what is salient) and perception (what is attended).}
\label{fig:cogstack}
\end{figure}

\subsection{Working Memory and Episodic Recall}
\label{sec:memcog}

Three findings from cognitive psychology shape our memory design.

\textbf{Bounded recency.} Working memory holds roughly seven plus or minus two
chunks in classic estimates \citep{miller1956magical}, but more recent
re-analyses suggest the active-manipulation capacity is closer to four
\citep{cowan2001magical}. We instantiate this bound directly: \sys's
\emph{recent} actions deque has a fixed maximum length of three, deliberately
chosen at the conservative end of the empirical range. The chosen length still
preserves the ``where am I in the local plan'' signal across short page
transitions; older steps are not deleted, but demoted to the persistent ledger
where they no longer occupy the live prompt.

\textbf{Multi-component structure.} Baddeley and Hitch's working-memory model
\citep{baddeley1974working,baddeley2000episodic} separates a central executive
from two slave systems---a phonological loop for verbal material and a
visuospatial sketchpad for spatial material---bound together by a later-added
\emph{episodic buffer} that integrates these streams with long-term knowledge.
\sys's three-role brain follows the same shape. The Planner is the central
executive (what to do next, given the task); the Worker uses both verbal
(action arguments, target labels) and visuospatial (\Vn{} bounding boxes,
viewport coordinates) channels; and the Ledger plays the role of the episodic
buffer, integrating both streams with the long-term task representation
(goal, plan, accumulated facts).

\textbf{Structured episodic recall.} People remember tasks not as raw sensory
streams but as schematic episodes: a goal, a sequence of attempts, a few
landmarks, and one or two ``do not do that again'' lessons
\citep{tulving1972episodic,sweller1988cognitive}. \sys's Ledger
(\S\ref{sec:ledger}) is the schematic episode for the current task: a single
mutable record with named slots for \code{goal}, \code{plan}, \code{facts},
\code{dead\_ends}, \code{recent}, \code{checkpoints}, and \code{episodic}. The
contrast with reflective-memory agents such as Reflexion
\citep{shinn2023reflexion} and Voyager \citep{wang2024voyager}---which
\emph{accumulate} reflections---is deliberate: we accumulate \emph{distilled}
items while \emph{discarding} their raw observation traces.

\textbf{Declarative vs. procedural memory.} A second axis cuts across
working-memory capacity. The classical taxonomy
\citep{squire1992declarative,anderson1982skill,anderson2004actr} separates
\emph{declarative} memory (facts and episodes one can verbalise) from
\emph{procedural} memory (skills and motor programs one executes without
re-deliberation). Fitts and Posner's three-stage model of skill acquisition
\citep{fitts1967human} predicts that with repetition, an action moves from
the cognitive stage (``find the button, plan the click'') through associative
to autonomous (``just click''); Logan's instance theory of automatization
\citep{logan1988automatic} gives a formal account of the same transition.
The implication for a web agent is that the perception--cognition stages
(vision call, planning) should not be re-paid on \emph{every} action. When
the page state has not changed materially, the action is procedural, and
the agent should fire it from a cached representation of where things are.
This is the architectural commitment behind \S\ref{sec:procedural}, our
DOM-cache subsystem.

\subsection{What Humans Actually Do When Browsing}
\label{sec:humanbrowse}

The cognitive primitives above are well-established in laboratory tasks, but
the load-bearing claim of this paper is about a specific environment: a person
in front of a real, cluttered, dynamic web page trying to get something done.
Four decades of empirical work on human web browsing converge on patterns that
directly motivate \sys's design.

\textbf{Scanpaths over rasters.} People do not perceive a webpage as a uniform
field of pixels; they construct it through a sequence of fixations and saccades
\citep{noton1971scanpath}. Each fixation lands on a candidate target chosen by
a mix of bottom-up saliency (luminance, contrast, motion) and top-down
relevance (does this match my goal). The Itti--Koch saliency-map model
\citep{itti1998saliency} formalises the bottom-up stream as a winner-take-all
selection over a topographic conspicuity map---essentially what the vision
model in \sys{} does when it emits \Vn{} bounding boxes. Crucially, scanpaths
on the \emph{same} page are largely consistent across users for a given task
\citep{josephson2002visual}: human attention is task-driven, not just
pixel-driven, which is also why our snapper must consult the expected label,
not only the geometry.

\textbf{The F-pattern and visual hierarchy.} Large-scale eye-tracking of
real-world web reading reveals strongly non-uniform attention. Nielsen's
``F-shaped pattern''---two horizontal sweeps near the top of the page followed
by a vertical sweep down the left edge---characterises how users skim
text-dense pages~\citep{nielsen2006fpattern,pernice2017eyetracking}. The
implication for an agent is that the visually dominant element is rarely the
intended target; the salient \emph{small} elements (search boxes, link
clusters, expand affordances) sit at predictable positions in the visual
hierarchy and are the things users actually click. This is exactly the
``small arrow beside a large label'' problem the chevron tiebreaker
(\S\ref{sec:chevron}) solves: the large text is what a naive
area-weighted snapper picks, but the small chevron is what a person picks.

\textbf{Information foraging and information scent.} Pirolli and Card's
information-foraging theory \citep{pirolli1995foraging,pirolli2007foraging}
models a user's web navigation as a forager moving between patches, choosing
patches by their perceived ``information scent''---proximal cues such as
link text, icons, and breadcrumbs that signal whether the distal information
goal is reachable. Two predictions follow. First, users \emph{do not
backtrack indefinitely}: when a patch's scent grows weak they move on, and
they remember the negative scent as a marker not to revisit. This is the
cognitive analogue of \sys's \code{dead\_ends} list, which records the URL
and cause of each failure so the agent does not re-enter the same patch.
Second, users \emph{do not retain the full content of patches they have
left}; they retain only the scent they extracted. This is the cognitive
analogue of our screenshot back-patch: the visual gist (the caption) survives,
the pixels do not.

\textbf{Recall is reconstructive, not photographic.} Outside the foveal
fixation point, visual memory of a page is famously schematic: people remember
the layout and one or two element identities but cannot reproduce most of what
they saw \citep{tulving1972episodic,pernice2017eyetracking}. An agent that
faithfully accumulates every screenshot is therefore not imitating the human
strategy; it is doing something the human visual system has explicitly evolved
\emph{not} to do, because such storage is expensive and the marginal
informational value of a five-page-old screenshot is near zero.

\subsection{Three Predictions}
\label{sec:predictions}

Cast as a theory, the cognitive analogy makes three predictions about how a
correctly designed agent should behave on real web tasks.

\begin{description}[leftmargin=*,itemsep=2pt]
  \item[P1 (Bounded growth).] Live-context token usage should grow
        sublinearly---ideally near constant---with the number of steps.
        Naive accumulation predicts linear growth; cognitive eviction
        predicts a slow log or constant trend. This is what information
        foraging theory predicts of human browsing memory consumption, too.
  \item[P2 (First-class failure).] Failures should not be merely tolerated;
        they should be remembered as \emph{negative landmarks}, analogous
        to the way a forager remembers a patch with weak scent and avoids
        it on later visits. Forgetting a failure within a few steps causes
        the agent to retry it.
  \item[P3 (Sub-element preference).] When the same screen region can be
        interpreted as one large clickable label or one small clickable
        sub-element (a chevron, a checkbox, an arrow), the small
        sub-element is usually the intended target---consistent with the
        F-pattern finding that user attention seeks small, semantically
        loaded affordances rather than dominant visual masses. A perception
        module that prefers larger area will systematically err on the
        ``United States $\blacktriangledown$'' family of UIs.
\end{description}

Our experiments (\S\ref{sec:eval}) test each prediction. \S\ref{sec:memory}
addresses P1 and P2; \S\ref{sec:action} addresses P3.

\section{System Architecture}
\label{sec:arch}

\sys{} consists of three cooperating LLM-driven roles, a vision pipeline, and a
tool layer (Figure~\ref{fig:threerole}). All roles share read access to a single
per-task Ledger; only the Worker mutates page state.

\begin{figure}[t]
\centering
\includegraphics[width=\linewidth]{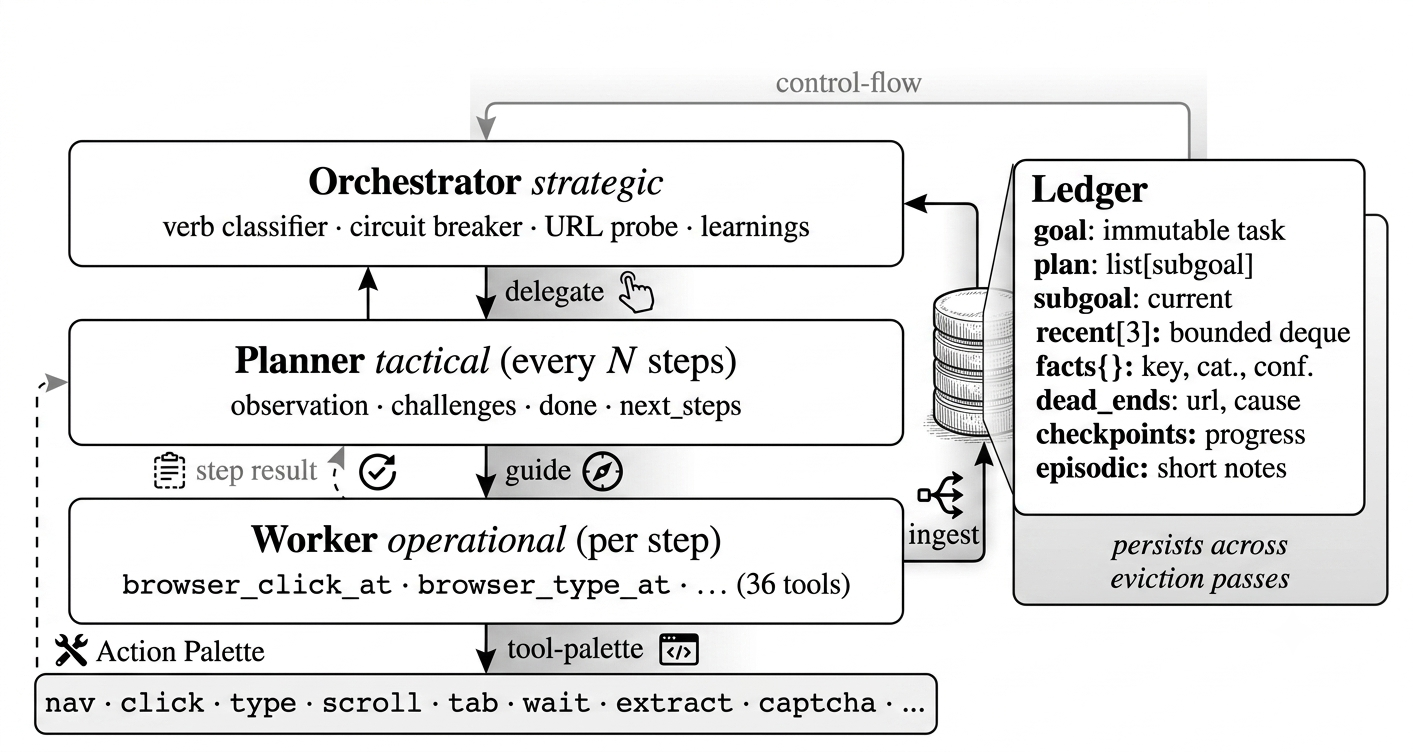}
\caption{Three-role brain. The Orchestrator classifies and routes; the Planner
re-evaluates progress every $N$ steps; the Worker emits per-step actions. The
structured Ledger is read-only for the Planner and the Orchestrator and is
mutated by the Worker via the memory hook.}
\label{fig:threerole}
\end{figure}

\subsection{Three Roles}
\label{sec:roles}

\textbf{Orchestrator.} The strategic layer. Receives a free-form user task,
classifies it (\S\ref{sec:routing}), and delegates to either a Search worker
(for read-only aggregation queries) or a Browser worker (for transactional and
visual tasks). It also maintains per-domain learnings---which engine tier
worked, which captcha strategy succeeded---and refuses redundant re-delegation
when a prior worker already returned a substantive result.

\textbf{Planner.} The tactical layer. Invoked every $N$ steps (default $N=4$)
and whenever the Worker signals possible completion. It receives the page
state, a condensed history, and the current Ledger render, and emits a JSON
object with \code{observation}, \code{challenges}, \code{done}, \code{next\_steps},
\code{final\_answer}, and \code{reasoning} fields. This is the only place
``the task is done'' can be declared.

\textbf{Worker.} The operational layer. Sees the current page state with
\Vn{}-labelled bounding boxes and decides one or more concrete actions per
step (default at most five before re-evaluation). Page-changing actions
(navigation, form submission) abort the action batch so the Planner can
re-engage.

Algorithm~\ref{alg:threerole} sketches the top-level loop.

\begin{algorithm}[t]
\SetAlgoLined
\DontPrintSemicolon
\KwIn{user task $T$, max steps $M$, planner interval $N$}
\KwOut{final answer or failure}
\BlankLine
$(\text{kind}, \text{conf}) \leftarrow \textsc{ClassifyTask}(T)$\;
\If{\textnormal{kind} = \textsc{search}}{
  \KwRet $\textsc{SearchWorker}(T)$\;
}
$\mathcal{L} \leftarrow \textsc{Ledger}(\text{goal}{=}T)$\;
$\textit{page} \leftarrow \textsc{OpenBrowser}(T.\text{url})$\;
\For{$s \leftarrow 1$ \KwTo $M$}{
  $\textit{state} \leftarrow \textsc{Observe}(\textit{page})$ \tcp*[r]{screenshot + DOM + \Vn{} prefetched cache}
  \If{$s \bmod N = 0 \;\textbf{or}\; \textit{workerDone}$}{
    $\textit{plan} \leftarrow \textsc{Planner}(\textit{state}, \mathcal{L}, T)$\;
    \lIf{$\textit{plan}.\text{done}$}{\KwRet $\textit{plan}.\text{finalAnswer}$}
  }
  $\textit{actions} \leftarrow \textsc{Worker}(\textit{state}, \mathcal{L}, \textit{plan}.\text{nextSteps})$\;
  \ForEach{$a \in \textit{actions}$}{
    $r \leftarrow \textsc{Execute}(a, \textit{page})$\;
    $\mathcal{L} \leftarrow \textsc{Ingest}(\mathcal{L}, a, r, \textit{state}.\text{url})$ \tcp*[r]{Alg.~\ref{alg:mem}}
    \lIf{$a$ is page-changing $\textbf{and}\; r.\text{success}$}{\textbf{break}}
  }
}
\KwRet \textit{failure}\;
\caption{Three-role run loop (top level).}
\label{alg:threerole}
\end{algorithm}

\subsection{Verb-Classifier Task Routing}
\label{sec:routing}

A surprisingly large fraction of browser-agent errors are caused not by bad
action choices but by the agent agreeing to perform the wrong \emph{kind} of
task---e.g., opening a browser session to answer a factual lookup
(``what is the capital of Norway''). The Orchestrator implements a deterministic
verb classifier (Algorithm~\ref{alg:routing}, full pseudocode in
Appendix~\ref{app:routing}) that scores the user instruction against five
pattern families and a date-indicator regex. Tasks scoring high on action
verbs, transactional patterns with a concrete date, or brand names are routed
to the browser; tasks scoring high on aggregation or factual lookup are routed
to a search worker that uses HTTP fetching only. A circuit breaker tracks
repeated delegations to the same \code{(domain, task)} hash and refuses
re-delegation when the prior result was substantive.

\subsection{Vision Pipeline}
\label{sec:vision}

Every screenshot is sent to a multimodal vision model (any
OpenAI-compatible vision endpoint will do; in practice a small,
cost-efficient flash-tier model is sufficient and what we deploy), which
returns a list of candidate interactive regions with bounding boxes and
short labels. The model emits boxes in a normalized
$[y_{\min}, x_{\min}, y_{\max}, x_{\max}]$ format scaled to $[0, 1000]$
against the screenshot, which the bridge denormalizes to CSS pixels.
Three engineering details make this practical at scale.

\textbf{Asynchronous prefetch.} Vision calls are scheduled immediately after a
mutating action; the LLM then issues the next action against the cached
response from the previous prefetch. This pipelining eliminates roughly
$900\text{ ms}$ per iteration from the critical path---about a 30\% wall-clock
reduction at typical step rates.

\textbf{DOM enrichment.} Each returned bounding box is annotated, at the bridge
layer, with information that is invisible to the vision model but essential for
correct action: the matching DOM selector index, \code{aria-expanded} state,
whether the element is currently active (checkbox checked, filter chip
toggled), the resolved \code{aria-labelledby} text, and the index of any
parent ``expand'' control. The snapper of \S\ref{sec:bbox} uses these to
disambiguate compound rows.

\textbf{Compound-row splitting.} When a single bounding box contains both a
text chip and an adjacent chevron (a common pattern in filter UIs), the bridge
splits the box into two before exposing it to the LLM, so the model can
address the chip and the chevron independently.

\textbf{Set-of-Marks anchoring across passes.} Before sending each new
screenshot to the vision model, the bridge overlays the bounding boxes
emitted on the \emph{previous} pass as dashed coloured rectangles
labelled \Vidx{1}, \Vidx{2}, $\dots$ \citep{yang2023setofmark}. The
model is instructed to (i) compute fresh bboxes from the current
screenshot and only then (ii) compare against the overlay, adopting
prior coordinates when they still match and dropping marks whose
underlying element has disappeared. This gives the vision tier
\emph{cross-turn perceptual continuity}---an analogue of how a human
maintains a stable mental map of where the search box was even as the
page below it scrolls---and reduces re-labelling drift on stable UI
elements. The overlay is suppressed when the prior pass is older than
ten seconds or the URL has changed, since stale anchors mislead more
than they help.

\textbf{Subgoal-aware bbox emphasis.} When the active subgoal is
known, the bridge forwards a short description of it (and any
declared lookup labels) to the vision model in the user prompt, with
the instruction that \code{intent\_relevant} be set true \emph{only}
for elements that serve this immediate subgoal---not for every
element that might serve the broader task. The downstream Worker
prompt re-renders bboxes ranked by subgoal relevance, so the
operational layer sees the locally salient affordances first.

\textbf{Page-type-aware vision tiering.} The vision model emits a
\code{page\_type} classification on every pass
(\code{search\_results}, \code{checkout\_form}, \code{product\_listing},
\code{map\_or\_booking}, \dots). Heavy pages---those with many filter
controls, long booking forms, dense result grids---tend to truncate or
under-emit on the smallest flash-tier model. Once a page is classified
as a complex type on a given \code{(session, url)}, subsequent passes
escalate directly to a larger fallback vision model, avoiding the
``flash truncates $\to$ compact retry $\to$ fallback'' round-trip ladder
that would otherwise repeat on every screenshot. The classification is
remembered per session, not globally, so the next task on the same URL
starts fresh.

\textbf{Coverage pass on validator-flagged misses.} Vision occasionally
culls a small, semantically critical control (icon buttons, single
checkboxes, value triggers) under the default $50$-bbox cap. When the
post-action validator reports a target label as missing from the
emitted bboxes but expected to be present, the bridge issues a
second-pass \emph{coverage} call: the bbox cap is lifted to $60$,
page-type culling rules are suspended, and the model is given the
explicit list of expected labels and any DOM-side coordinate hints.
This recovers $\sim 5\%$ of otherwise-failed clicks in our internal
traces without inflating the average per-screenshot vision cost,
because coverage passes only fire on validator-flagged misses.

\subsection{Tool Layer}
\label{sec:tools}

The Worker has access to thirty-six browser tools spanning navigation,
interaction, scroll, tab management, control flow, content extraction, captcha
handling, and geo-blocking detection. Appendix~\ref{app:tools} tabulates the
complete list. The two most important for present purposes are
\code{browser\_click\_at(vision\_index, target\_label)}, which clicks within a
\Vn{} bounding box and triggers the snapper, and
\code{browser\_type\_at(vision\_index, text)}, which routes typing through a
text-fixing pre-pass to normalise dates, numbers, and currency formats
expected by site validators.

\subsection{Procedural Click: DOM Cache and Asynchronous Vision}
\label{sec:procedural}

A naive vision-first agent calls a vision model on every screenshot, every
click, every type---each call costing $\sim 800$ tokens and $3$--$5$ seconds
of vision-model latency. A human, by contrast, does not re-fixate the
\emph{Search} button between clicks; once the location is encoded, the next
press is a procedural action drawn from motor memory
\citep{anderson1982skill,logan1988automatic,fitts1967human}. \sys{}
realises the same separation: vision is the declarative path, DOM-cache
clicks are the procedural path, and a gating predicate decides which is
appropriate on each step.

\paragraph{Two click paths.}
The Worker exposes both:
\begin{description}[leftmargin=*,itemsep=2pt,topsep=2pt]
  \item[Declarative (vision-grounded):]
        \code{browser\_click\_at(vision\_index, target\_label)} resolves a
        \Vn{} bounding box against the cached \emph{vision epoch}---a frozen
        snapshot of the last vision-model response that the LLM saw in its
        message history. The snapper (\S\ref{sec:bbox}) maps the bbox to a
        viewport coordinate and the cascade (\S\ref{sec:cascade}) dispatches
        the click.
  \item[Procedural (DOM-direct):]
        \code{browser\_click(index)} and
        \code{browser\_click\_selector(css, target\_label)} reference a
        DOM-selector entry directly. No vision call is involved; the click
        flows straight from a cached selector map to the click cascade.
        This is the fast, cheap path used when ``where to click'' is
        already known.
\end{description}

\paragraph{The gating predicate.}
After every mutating action, the bridge computes three lightweight
fingerprints of the page state and compares them to the values captured at
the last vision call:
\begin{itemize}[leftmargin=*,itemsep=2pt,topsep=2pt]
  \item \textbf{DOM hash} $h_{\text{dom}}$: a SHA-1 over the interactive
        element list (count, tag histogram, sample of \texttt{aria-label}
        text). Sensitive to elements appearing, disappearing, or
        reordering.
  \item \textbf{DOM-text hash} $h_{\text{text}}$: a finer hash over the
        rendered text and a 100-pixel-bucketed scroll position. Sensitive
        to scroll movements and in-place text mutations.
  \item \textbf{Iframe signature} $h_{\text{frame}}$: a hash over iframe
        descendants, busted when iframe contents mutate without a
        top-level DOM change (e.g.~quiz advances, calculator updates).
\end{itemize}
Combined with a \texttt{mutation\_delta} counter (the number of DOM nodes
added/removed since last screenshot) and the current URL, these form the
\emph{cache key}. If the key matches, the previous vision epoch is still
authoritative, and a procedural click may proceed with no new perception;
if the key has drifted past tolerance, the agent must take a fresh
screenshot (which triggers a fresh vision pass), or wait for a background
prefetch to complete.

\paragraph{Asynchronous vision prefetch.}
A successful mutating action does not block on vision. Instead it
\emph{schedules} a background vision call to refresh the vision epoch for
whichever action comes next. Two timeouts gate the next click that needs
vision:
\begin{itemize}[leftmargin=*,itemsep=2pt,topsep=2pt]
  \item \textbf{Soft sync} ($1{,}500$~ms by default): if the prefetch
        completes within this window, the next click uses fresh bboxes; if
        not, the click proceeds on the previous epoch and the bridge
        annotates the result with a \texttt{[vision\_lag]} marker so the
        Worker can decide whether to re-screenshot.
  \item \textbf{Hard sync} ($8{,}000$~ms): used before critical actions
        (e.g.~a form submit on a page the agent has never seen) where
        stale vision is unacceptable. The click blocks until the prefetch
        returns or the timeout fires.
\end{itemize}
This pipelining means vision latency is paid \emph{in parallel} with the
LLM's next reasoning step, not in series with it.

\paragraph{Dead-click guard and toggle awareness.}
The DOM cache, like procedural motor memory, can become \emph{wrong}: a
button can be removed, an input disabled, a filter pre-applied. To prevent
the agent from re-firing the same action repeatedly with no effect, a
dead-click guard records the target XPath and resulting
\texttt{mutation\_delta} of each click; if the same target is fired twice
in a row with \texttt{mutation\_delta}~$= 0$, the second attempt is
refused and the Worker is told to re-screenshot. A toggle exemption fires
when the element's \texttt{is\_active} state flipped between attempts---a
filter being un-applied looks like no DOM mutation but is in fact a
legitimate action.

\begin{algorithm}[t]
\SetAlgoLined
\DontPrintSemicolon
\KwIn{action $a$, page state $s$, vision epoch $E$, prior fingerprint $\phi^{\text{prev}}$}
\KwOut{click executed or deferred}
\BlankLine
$\phi \leftarrow (h_{\text{dom}}(s),\, h_{\text{text}}(s),\, h_{\text{frame}}(s),\, u(s))$ \tcp*[r]{cache key}
$\Delta \leftarrow s.\text{mutation\_delta}$\;
\BlankLine
\uIf{\textnormal{$a$ is a procedural click}\,\textbf{and}\,$\phi = \phi^{\text{prev}}$\,\textbf{and}\,$\Delta < \theta$}{
  \tcp{Page state matches the vision epoch; click from procedural memory.}
  $e \leftarrow E.\text{selectorMap}[a.\text{index}]$\;
  \KwRet $\textsc{Dispatch}(e,\, \text{cascade})$ \tcp*[r]{no vision call}
}
\Else{
  \tcp{Page has drifted from the epoch; need fresh perception.}
  \uIf{$\textsc{PrefetchPending}(E)$}{
    $\textit{ok} \leftarrow \textsc{WaitOrTimeout}(E,\, \tau_{\text{soft}}{=}1500\,\text{ms})$\;
    \lIf{$\neg\,\textit{ok}$ \textnormal{and} $a$ \textnormal{is critical}}{$\textit{ok} \leftarrow \textsc{WaitOrTimeout}(E,\, \tau_{\text{hard}}{=}8000\,\text{ms})$}
  }
  \lIf{$\neg\,\textit{ok}$}{annotate result with $[\text{vision\_lag}]$}
  $E \leftarrow \textsc{LatestEpoch}()$\;
  \tcp{Now resolve via vision-grounded bbox snap (Alg.~\ref{alg:bbox}).}
  $\text{bbox} \leftarrow E.\text{lookup}(a.\text{vision\_index},\, a.\text{target\_label})$\;
  $(x,y) \leftarrow \textsc{SnapToInteractive}(\text{bbox})$\;
  \KwRet $\textsc{Dispatch}((x,y),\, \text{cascade})$\;
}
\BlankLine
\textbf{after a successful click}:\;
\Indp
\lIf{\textnormal{dead-click guard fired}}{refuse and request screenshot}
$\textsc{Schedule}(\textsc{VisionPrefetch}(s'))$ \tcp*[r]{async; overlaps with next LLM turn}
$\phi^{\text{prev}} \leftarrow \phi(s')$\;
\Indm
\caption{Vision-gated click: declarative vs.\ procedural path.}
\label{alg:vgc}
\end{algorithm}

\paragraph{Empirical effect.}
On a representative twenty-step shopping task the naive policy makes one
vision call per action (twenty calls). \sys{}'s gating reduces
this to roughly one vision call per three to five actions---five to seven
vision calls over the same trajectory---without measurable success-rate
loss. The savings compound: fewer vision tokens, lower wall-clock latency
(prefetch overlaps with LLM inference), and a smaller cache-invalidating
delta on each step, which itself improves the prompt-cache hit rate
(\S\ref{sec:memempirics}).

\section{Human-Inspired Memory}
\label{sec:memory}

This section is the centrepiece of the paper. We describe the Ledger
(\S\ref{sec:ledger}), the six-phase eviction loop (\S\ref{sec:eviction}), the
role-sliced views (\S\ref{sec:roleview}), the empirical behaviour
(\S\ref{sec:memempirics}), and the relationship to other memory-aware agents
(\S\ref{sec:memcompare}).

\subsection{Why Naive Accumulation Fails}
\label{sec:naivefail}

A typical browser-agent loop appends a screenshot ($\approx$1--3K tokens), an
element list ($\approx$30--80 rows), the prior tool result, and a state block on
every iteration. By iteration 20 the live context is dominated by stale
observation payload; the cache hit rate drops from $\sim$90\% on iteration~2 to
below $\sim$15\% by iteration~20; and---most damagingly---the model's
attention is increasingly drawn to obsolete state that no longer reflects the
page. The left panel of Figure~\ref{fig:memcompare} shows this effect on a
representative twenty-step task. The right panel shows the same trace under our
eviction loop: tokens stay near a fixed point.

\begin{figure}[t]
\centering
\includegraphics[width=\linewidth]{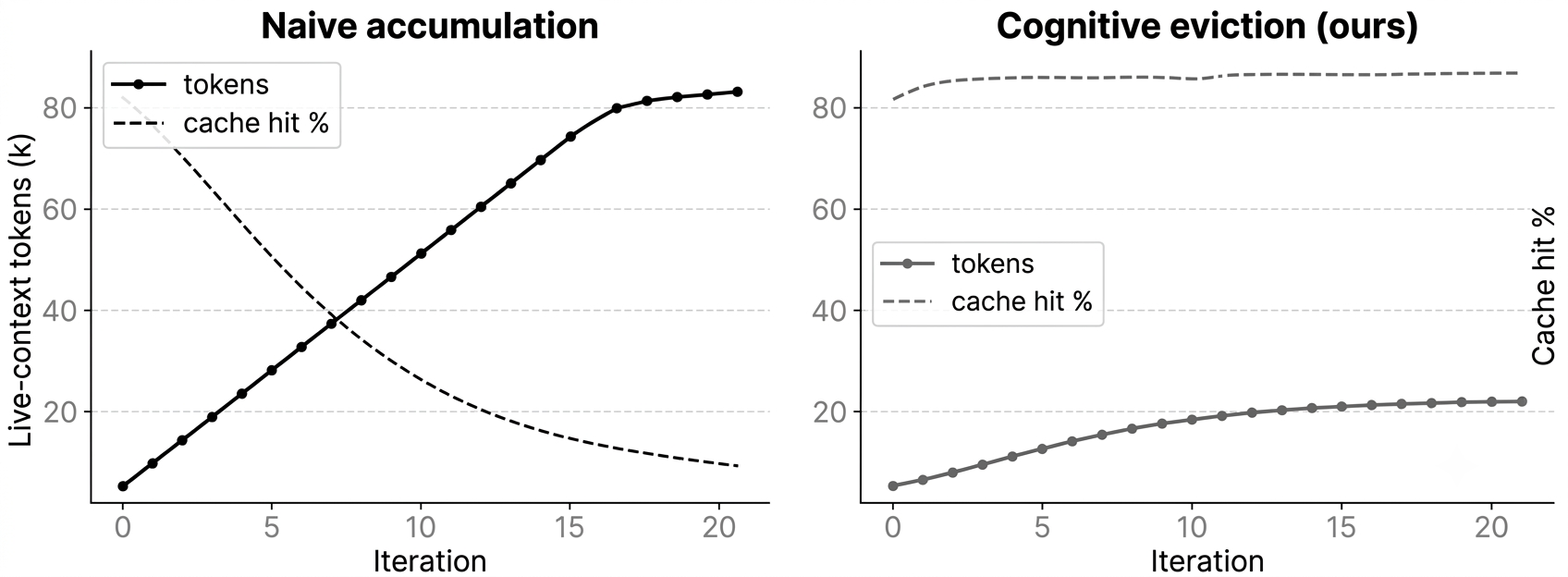}
\caption{Naive accumulation (left) versus cognitive eviction (right) on a
representative twenty-step task. Solid lines show live-context tokens; dashed
lines show prompt-cache hit rate. Naive accumulation roughly triples context
size and collapses cache hit by iteration~15; cognitive eviction holds both
near a fixed point.}
\label{fig:memcompare}
\end{figure}

\subsection{The Structured Ledger}
\label{sec:ledger}

The Ledger is a single per-task data structure with the following slots:

\begin{itemize}[leftmargin=*,itemsep=2pt,topsep=2pt]
  \item \code{goal}: the immutable task description (e.g.\ ``book a 3-day stay
        in Khulna on gozayaan.com, check-in April 23, under \$40/night'').
  \item \code{plan}: a list of subgoals, set at start by the Planner and
        revisable across the task.
  \item \code{subgoal}: the currently active subgoal.
  \item \code{recent}: a \textbf{bounded deque of length 3} of the most recent
        \code{StepOutcome} records. Older steps fall off the front
        automatically but are archived to disk.
  \item \code{facts}: a dictionary of structured key--value observations, each
        tagged with category (\code{observation}, \code{constraint},
        \code{preference}, \code{identity}, \code{credential}, \code{derived}),
        confidence, and source step. The most-recently-referenced facts are
        rendered first when the Ledger is read.
  \item \code{dead\_ends}: a list of negative landmarks, each with a one-line
        description, the URL where the failure happened, and a cause tag
        (\code{network}, \code{bot\_block}, \code{stale\_selector},
        \code{postcondition\_miss}, \code{policy\_block}, \code{rate\_limit},
        \code{handoff\_timeout}, \code{generic}).
  \item \code{checkpoints}: progress landmarks (navigation, login, filter
        applied, item added to cart) used to detect regressions when the
        agent navigates backward.
  \item \code{episodic}: free-form short notes for subgoal summaries and
        archival markers.
\end{itemize}

The Ledger is persisted incrementally: \code{ledger.json} is rewritten on every
update; \code{steps.jsonl}, \code{facts.jsonl}, and \code{episodic.jsonl} are
append-only. Replay and post-mortem analysis can reconstruct the entire task
trajectory from these files even if the live context has long since evicted
the relevant turns.

\subsection{Six-Phase Eviction Loop}
\label{sec:eviction}

Before every LLM call the Memory Hook runs Algorithm~\ref{alg:mem}, a sequence
of six idempotent passes over the live message list. Each phase corresponds to
a cognitive primitive:

\begin{algorithm}[t]
\SetAlgoLined
\DontPrintSemicolon
\KwIn{message list $M$, ledger $\mathcal{L}$}
\KwOut{evicted $M'$, updated $\mathcal{L}'$ (in place)}
\BlankLine
\tcp{Phase 1: keep last 2 screenshots, replace older with text marker}
$M \leftarrow \textsc{BackPatchScreenshots}(M, k{=}2)$\;
\tcp{Phase 2: keep last failure; collapse older into one-line dead-ends}
$\textit{collapsed} \leftarrow \textsc{CollapseFailedTools}(M, k{=}1)$\;
\ForEach{$c \in \textit{collapsed}$}{
  $\mathcal{L}.\textsc{MarkDeadEnd}(c.\text{reason}, c.\text{url}, c.\text{cause})$\;
}
\tcp{Phase 3: strip stale [SESSION\_STATE] blocks, keep same-URL ``stuck'' signal}
$M \leftarrow \textsc{CollapseStaleStateBlocks}(M, k{=}2)$\;
\tcp{Phase 4: strip thinking/reasoning blocks from old assistants}
$M \leftarrow \textsc{StripStaleThinkingBlocks}(M, k{=}3)$\;
\tcp{Phase 5: above 30 msgs, hard-evict content older than last 5 turns}
\If{$|M| > 30$}{
  $M \leftarrow \textsc{GutOldContent}(M, \text{keepTurns}{=}5)$\;
  $\mathcal{L}.\textsc{Note}(\text{``archived \_ messages at iter \_; ledger remains authoritative''})$\;
}
\tcp{Phase 6: keep last element list, collapse older}
$M \leftarrow \textsc{CollapseElementLists}(M, k{=}1)$\;
\BlankLine
$M \leftarrow \textsc{RefreshLedgerInSystemMessage}(M, \mathcal{L}.\textsc{RenderForLLM}())$\;
\KwRet $(M, \mathcal{L})$\;
\caption{Cognitive memory eviction (runs before every LLM call).}
\label{alg:mem}
\end{algorithm}

\begin{description}[leftmargin=*,itemsep=2pt,topsep=2pt]
  \item[Phase 1: Screenshot back-patch.] Keep only the two most recent
        screenshots; replace older image blocks with a single text marker
        ``[screenshot from prior turn evicted]''. Vision captions (the
        sibling text blocks that summarise what was on screen) are preserved,
        so the model still has access to the \emph{interpretation} of an old
        screen, just not its pixels. This mirrors the way human visual memory
        retains schematic gist long after raw imagery has decayed.
  \item[Phase 2: Failure collapse.] Keep only the most recent failed tool
        response verbatim; older failures collapse to a one-line cause and
        are simultaneously demoted to \code{dead\_ends} entries. This
        operationalises prediction P2 from \S\ref{sec:predictions}: failures
        are remembered, just not as full transcripts.
  \item[Phase 3: State-block eviction.] Strip \code{[SESSION\_STATE\ldots]}
        blocks from messages older than the last two turns, with one
        exception: state blocks at the same \code{(url, title)} as the
        current page are preserved as a ``you are stuck'' signal.
  \item[Phase 4: Thinking-block strip.] Clear the
        \code{thinking\_blocks} / \code{reasoning\_content} payloads of
        assistant messages older than the most recent three. Tool calls
        themselves are never stripped---this preserves the tool-use /
        tool-result pairing that the LLM API requires.
  \item[Phase 5: Old-message gut.] When the live message count exceeds 30,
        hard-evict the \code{content} of messages older than the last five
        turns, marking them \code{\{"\_archived": true\}}. The message
        \emph{count} is preserved (so downstream tooling that indexes by turn
        number still works), but the bodies become ``[archived]''. A single
        episodic note is appended to the Ledger documenting the eviction.
  \item[Phase 6: Element-list collapse.] Keep only the most recent
        \code{[ELEMENTS \dots]} block (typically the output of
        \code{browser\_list\_elements}); older blocks collapse to a one-line
        marker advising the model to re-list if needed.
\end{description}

After the six phases, the Hook refreshes the rendered Ledger inside the system
message. Because the Ledger sits in a stable block at the top of the prompt,
the system-message prefix remains cache-able even as it is updated; the
prompt-cache hit rate stays near $\sim$87\% across our traces.

\subsection{Declarative versus Procedural Tiers}
\label{sec:tiers}

The Ledger described above is the agent's \emph{declarative} memory: facts,
plans, dead-ends, episodic notes---all things the LLM can reason about in
words. \sys{} keeps a separate \emph{procedural} tier in the DOM-cache
subsystem (\S\ref{sec:procedural}): cached selector maps, vision-epoch
bboxes, and the page-state fingerprints that decide when a cached action
is still safe to fire. This tier is invisible to the LLM at the level of
words---the model does not ``recall'' selector indices---but it is what
turns most clicks into one-step procedural emissions rather than three-step
perceive-plan-act cycles. The split mirrors the cognitive distinction
between declarative knowledge (``the cancel button is the red one bottom
right'') and procedural skill (``my hand goes there'')
\citep{squire1992declarative,anderson2004actr}.

\subsection{Role-Sliced Views}
\label{sec:roleview}

The Ledger is rendered differently for different consumers. The
Orchestrator sees the full Ledger (goal, plan, all facts, all dead-ends, all
checkpoints) because routing decisions depend on the whole picture. The
Worker sees a sliced view: the active subgoal, the three most recent
\code{StepOutcome}s, the five most relevant facts (by most-recent-reference),
and the dead-ends scoped to the current URL. This is the analogue of
attentional gating: the operational layer is shielded from strategic context
it does not need.

\subsection{Empirical Behaviour}
\label{sec:memempirics}

In Figure~\ref{fig:memcompare} the right panel reports the actual measured
trace on a representative shopping task: live-context tokens stay between
$\sim$11K and $\sim$22K across twenty steps, with a slow drift caused by
genuine new facts being added to the Ledger; cache hit rate stays above
80\% throughout. By contrast a naive control run on the same task reaches
$\sim$80K tokens by iteration~20 with cache hit below 15\%. The relative
token saving is $\sim$50\% per iteration averaged across our internal
evaluation set, and is larger for longer tasks.

\subsection{Relationship to Reflective Memory}
\label{sec:memcompare}

Reflexion \citep{shinn2023reflexion} and Voyager \citep{wang2024voyager}
\emph{add} to the agent's memory at every turn---a self-critique, a learned
skill. \sys{} \emph{subtracts} at every turn while structuring what remains.
The two approaches are orthogonal: one could feed Reflexion-style critiques
into the Ledger's \code{episodic} list, or have Voyager-style skills be
recalled as facts. We do not test that combination here.

\section{Action Execution}
\label{sec:action}

Three components turn a Worker's high-level intent into a click that a real
website accepts: a bounding-box snapper that converts a \Vn{} reference into
viewport coordinates, a three-tier cascade that dispatches the click, and a
humanization layer that gives motion a plausible kinematic profile.

\subsection{Bounding-Box Snapping}
\label{sec:bbox}

A vision-emitted bounding box rarely lines up exactly with a single clickable
element. It may enclose multiple candidates, miss the true target by a few
pixels, or contain both a label and a small expand control. Our snapper
(Algorithm~\ref{alg:bbox}) is a three-phase resolver.

\begin{figure}[t]
\centering
\includegraphics[width=\linewidth]{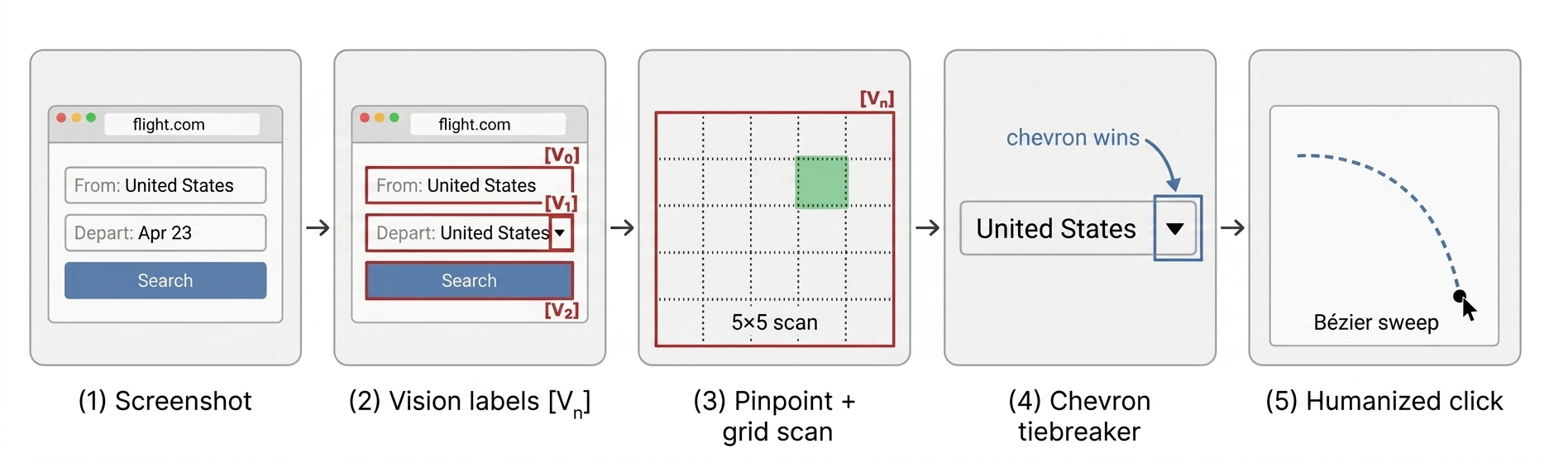}
\caption{Vision-to-action pipeline. (1) screenshot, (2) vision model emits
\Vn{} bboxes, (3) snapper pinpoint scan at the centre, (4) on compound rows
(e.g.\ \emph{United States}~$\blacktriangledown$) the chevron tiebreaker
prefers the small expand sub-element, (5) humanized cursor follows a B\'ezier
to the chosen point.}
\label{fig:bboxpipe}
\end{figure}

\begin{algorithm}[t]
\SetAlgoLined
\DontPrintSemicolon
\KwIn{vision bbox $B = (x_0, y_0, x_1, y_1)$, expected label $\ell$ (optional)}
\KwOut{$(x, y, \text{snapped}, \text{xpath})$ or failure}
\BlankLine
$(c_x, c_y) \leftarrow ((x_0+x_1)/2, (y_0+y_1)/2)$\;
\tcp{Phase A: pinpoint at centre}
$\textit{stack} \leftarrow \texttt{document.elementsFromPoint}(c_x, c_y)$\;
$e \leftarrow \textsc{WalkFront2Back}(\textit{stack}, \text{interactive})$\;
\If{$e$ \textnormal{is an} \texttt{<iframe>} \textnormal{(Phase A-alt)}}{
  \tcp{Same-origin iframe descent}
  \If{$e.\texttt{contentDocument}$ \textnormal{accessible}}{
    $B' \leftarrow B - (e.\text{rect}.\text{left}, e.\text{rect}.\text{top})$\;
    $e' \leftarrow \textsc{PinpointInner}(e.\texttt{contentDocument}, B')$\;
    \If{$e' \ne \emptyset$}{
      $(c_x, c_y) \leftarrow \textsc{TranslateToViewport}(e'.\text{rect}, e.\text{rect})$\;
      $e \leftarrow e'$ ; \textnormal{record iframe chain}\;
    }
  }
}
\If{$\ell$ \textnormal{provided and} $|\ell| \ge 3$}{
  \If{$\textsc{LabelMatch}(e, \ell)$}{
    \KwRet $(c_x, c_y, \text{true}, e.\text{xpath})$\;
  }
}
\tcp{Phase B: $5{\times}5$ grid scan with chevron tiebreaker}
$\textit{best} \leftarrow \emptyset$;\;
\For{$i \leftarrow 1$ \KwTo $4$}{
  \For{$j \leftarrow 1$ \KwTo $4$}{
    $(p_x, p_y) \leftarrow (x_0 + (x_1-x_0)\frac{i}{5},\; y_0 + (y_1-y_0)\frac{j}{5})$\;
    $h \leftarrow \textsc{FirstInteractiveAt}(p_x, p_y)$\;
    \lIf{$h = \emptyset$}{\textbf{continue}}
    $a \leftarrow \textsc{IntersectionArea}(h.\text{rect}, B)$\;
    $\lambda \leftarrow \textsc{LabelScore}(h, \ell)$ \tcp*[r]{leniency by widget type}
    $\chi \leftarrow \textsc{ChevronWeight}(h, B)$ \tcp*[r]{aria-expanded, $\blacktriangledown\blacktriangleright\blacktriangleleft\vdots$, ``expand''}
    $\sigma \leftarrow a \cdot \lambda \cdot \chi$\;
    \lIf{$\sigma > \textit{best}.\sigma$}{$\textit{best} \leftarrow (h, p_x, p_y, \sigma)$}
  }
}
\lIf{$\textit{best} = \emptyset$}{\KwRet $\textit{failure}$}
\KwRet $(\textit{best}.p_x, \textit{best}.p_y, \text{true}, \textit{best}.h.\text{xpath})$\;
\caption{Vision-bbox snap with chevron tiebreaker.}
\label{alg:bbox}
\end{algorithm}

\textbf{Phase A (pinpoint).} Sample the bounding-box centre and walk
\code{document.elementsFromPoint(cx,cy)} front-to-back, taking the first
interactive ancestor. If the expected label matches the snapped element's text
or \code{aria-label}, return the centre directly.

\textbf{Phase A-alt (iframe descent).} If the element at the centre is an
iframe, attempt to access its \code{contentDocument}. On a same-origin iframe
we recompute the bounding box in iframe-local coordinates, repeat the
pinpoint scan inside the inner document, and translate the result back to
viewport space. On a cross-origin iframe we flag the case so the Worker can
escalate to a Puppeteer Frame walk.

\textbf{Phase B (grid scan).} If the centre falls on empty space or the
expected label does not match, sample a $5\times5$ grid across the bounding
box, score each interactive candidate by intersection area $\times$ label
match $\times$ chevron weight, and return the highest-scoring point. The
composite score lets a small but well-labelled candidate beat a large but
unlabelled one.

\subsection{The ``Small Arrow Beside a Large Label'' Problem}
\label{sec:chevron}

A recurring failure mode in vision-grounded agents is the compound row: a
filter row whose visible text is, say, \emph{``United States''}, and whose
right edge contains a tiny ``$\blacktriangledown$'' chevron that, when clicked,
expands a list of specific regions. The vision model returns a single bounding
box spanning the whole row; a naive snapper clicks the centre, lands on the
text label, and nothing happens---the label itself is not interactive, only
the chevron is. The agent re-screenshots, sees no change, retries, and loops.

Our snapper resolves this with a \emph{chevron tiebreaker}. On row-shaped
bounding boxes (width~$\ge$~60~px, height~$\ge$~24~px) we boost the score of
candidates whose attributes or text content match an expand-control pattern:
the \code{aria-expanded} attribute is present, \code{aria-haspopup} is set,
the text content contains one of the chevron glyphs
$\blacktriangledown\blacktriangleright\blacktriangleleft\blacktriangle$
(or their hollow variants), or the \code{aria-label} mentions ``expand'',
``collapse'', ``toggle'', or ``more''. The boost is multiplicative (factor up
to $3$) but capped so that it only wins when the intersection area is within
$30\%$ of the area-only leader; this preserves the leader for non-compound
rows. The same mechanism also handles compound-row \emph{splitting} at the
vision-pipeline layer: when the vision model merges a chip and its
chevron into one box, the bridge splits the box into two before exposing
it to the LLM (\S\ref{sec:vision}). Predictably, this fixes prediction P3 from
\S\ref{sec:predictions} in the cases we observe.

\subsection{Three-Tier Click Cascade}
\label{sec:cascade}

A click that lands at the right viewport coordinate is necessary but not
sufficient: it must also be dispatched in a way the page accepts. Many
modern sites refuse non-trusted clicks (\code{event.isTrusted = false}) or
require a real \code{mousedown}/\code{mouseup} pair, not a synthetic
\code{el.click()}. We therefore use a three-tier cascade
(Figure~\ref{fig:cascade}).

\begin{figure}[t]
\centering
\includegraphics[width=\linewidth]{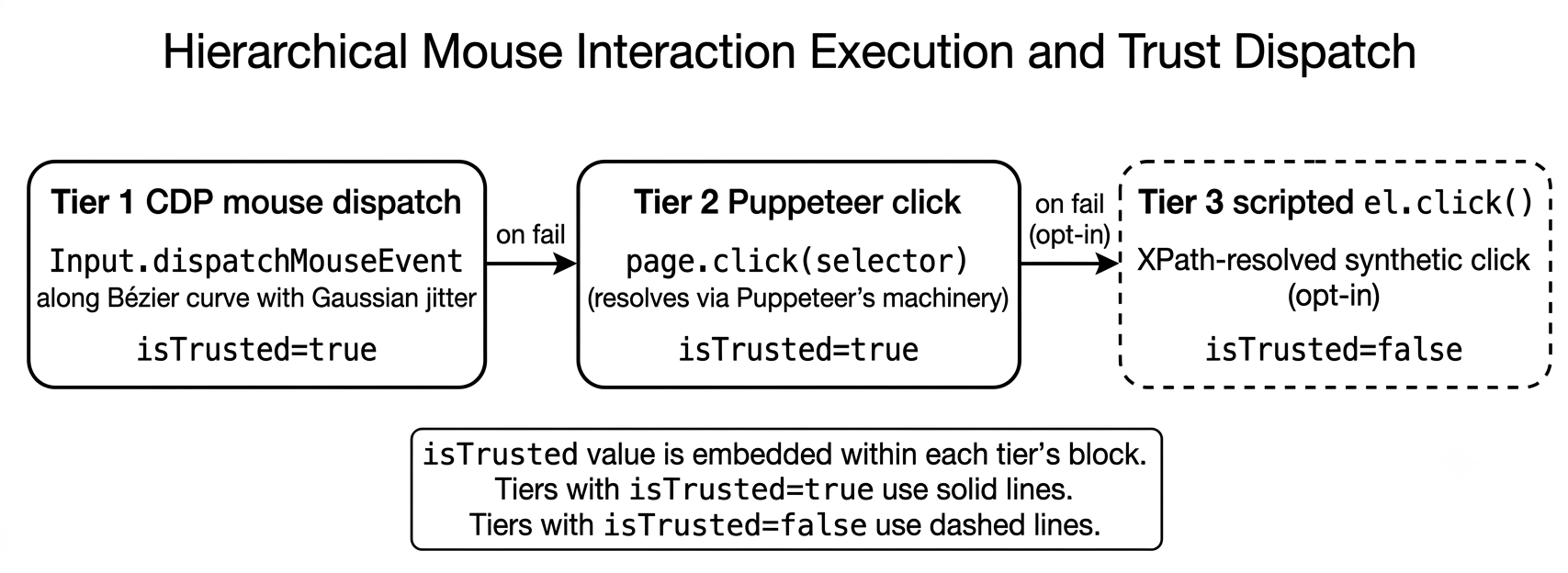}
\caption{Three-tier click cascade. Tier~1 dispatches via the Chrome DevTools
Protocol with a humanized B\'ezier sweep, preserving \code{isTrusted=true}.
Tier~2 falls back to Puppeteer's selector-based \code{page.click}, still
trusted. Tier~3 is a scripted \code{el.click()} via XPath; it is opt-in only
because \code{isTrusted=false} is a bot-detection signal.}
\label{fig:cascade}
\end{figure}

\textbf{Tier 1: CDP mouse dispatch.} The default. Resolves the target element
via a CSS selector, verifies via \code{document.elementsFromPoint} that the
intended element is on top, sweeps the cursor along a B\'ezier curve with
Gaussian jitter from the previously known cursor position
(Algorithm~\ref{alg:humotion}), and dispatches a
\code{mousePressed}/\code{mouseReleased} pair via CDP. \code{isTrusted=true}.

\textbf{Tier 2: Puppeteer click.} Triggered if Tier~1 fails for selector or
visibility reasons. Calls Puppeteer's \code{page.click(selector)}, which
also produces a trusted click but resolves the element via Puppeteer's own
machinery rather than ours; useful when our selector differs.

\textbf{Tier 3: scripted \code{el.click()}.} Last resort, opt-in via an
explicit flag. Resolves the element by XPath and calls
\code{el.click()} synchronously. \code{isTrusted=false}, which is a
detectable signal on bot-aware sites; for this reason Tier~3 is disabled by
default on domains whose learnings file marks them as bot-protected.

Algorithm~\ref{alg:cascade} formalises the cascade. After every attempt the
agent records which tier worked and feeds it back to the per-domain
learnings store, so subsequent tasks on the same domain start at the
empirically best tier.

\begin{algorithm}[t]
\SetAlgoLined
\DontPrintSemicolon
\KwIn{target $t$ (CSS selector or $(x,y)$ coords), policy flags (humanize, allowUntrusted)}
\KwOut{$\textsc{ToolResult}\{\,\text{success},\, \text{tried}[\,],\, \text{reason},\, \text{alternatives}\,\}$}
\BlankLine
$\textsc{Tried} \leftarrow [\,]$\;
\BlankLine
\tcp{\textbf{Tier 1}: CDP mouse dispatch (\texttt{isTrusted}=\textsc{true})}
$\textsc{Tried}.\text{append}(\texttt{cdp})$\;
$p \leftarrow \textsc{ProbeElement}(t)$\;
\lIf{$\neg\,p.\text{found} \;\vee\; p.\text{disabled} \;\vee\; \neg\,p.\text{visible}$}{\KwRet $(\bot,\, \textsc{Tried},\, p.\text{reason})$}
\lIf{$\neg\,p.\text{inViewport}$}{$\textsc{ScrollIntoView}(t,\,\text{block}{=}\text{centre})$\,;\, \textsc{Wait}($150\,\text{ms}$)}
$(c_x, c_y) \leftarrow \textsc{Centre}(p.\text{rect})$\;
$v \leftarrow \textsc{ValidateTopMost}(c_x, c_y, t)$ \tcp*[r]{\texttt{elementsFromPoint} stack check}
\lIf{$v = \texttt{covered}$}{\KwRet $(\bot,\, \textsc{Tried},\, \texttt{element\_covered},\, \textsc{Suggest}(t))$}
\lIf{$v = \texttt{ambiguous}$}{\KwRet $(\bot,\, \textsc{Tried},\, \texttt{selector\_ambiguous})$}
\If{$\textnormal{humanize}$}{
  $\textsc{Path} \leftarrow \textsc{HumanizedMotion}(\textsc{LastCursor}(),\, (c_x, c_y))$ \tcp*[r]{Alg.~\ref{alg:humotion}}
  $\textsc{ReplayPath}(\textsc{Path})$\;
  $\textsc{Wait}(\mathcal{N}(100, 50)\,\text{ms})$\;
  $(c_x, c_y) \leftarrow (c_x, c_y) + \mathcal{N}(\mathbf{0}, 2^2 \mathbf{I})$ \tcp*[r]{$\pm 2\,$px centring jitter}
}
$\textsc{CDP.dispatch}(\texttt{mousePressed},\, (c_x, c_y))$\;
$\textsc{Wait}(\mathcal{N}(90, 25)\,\text{ms})$\;
$\textsc{CDP.dispatch}(\texttt{mouseReleased},\, (c_x, c_y))$\;
\lIf{$\textsc{MutationDelta}() > 0 \;\vee\; \textsc{UrlChanged}()$}{\KwRet $(\top,\, \textsc{Tried},\, \emptyset)$}
\BlankLine
\tcp{\textbf{Tier 2}: Puppeteer selector click (\texttt{isTrusted}=\textsc{true})}
$\textsc{Tried}.\text{append}(\texttt{puppeteer})$\;
$\textsc{ok} \leftarrow \textsc{Puppeteer.waitForSelector}(t,\, \text{timeout}{=}5\,\text{s})$\;
\lIf{$\neg\,\textsc{ok}$}{\KwRet $(\bot,\, \textsc{Tried},\, \texttt{stale\_selector},\, \textsc{Suggest}(t))$}
$\textsc{Puppeteer.click}(t)$\,;\, $\textsc{Wait}(1.5\,\text{s})$\;
\lIf{$\textsc{MutationDelta}() > 0$}{\KwRet $(\top,\, \textsc{Tried},\, \emptyset)$}
\BlankLine
\tcp{\textbf{Tier 3}: scripted \texttt{el.click()} (\texttt{isTrusted}=\textsc{false}; opt-in)}
\uIf{$\textnormal{allowUntrusted}$}{
  $\textsc{Tried}.\text{append}(\texttt{js})$\;
  $e \leftarrow \textsc{ResolveXPath}(t)$\;
  \lIf{$\neg\,\textsc{InViewport}(e)$}{$\textsc{ScrollIntoView}(e,\, \text{block}{=}\text{nearest})$}
  $e.\text{click}()$ \tcp*[r]{warning: anti-bot signal}
  \lIf{$\textsc{MutationDelta}() > 0$}{\KwRet $(\top,\, \textsc{Tried},\, \emptyset)$}
}
\BlankLine
$\textsc{LearningsStore}.\textsc{Record}(\textsc{Domain}(),\, \textsc{Tried},\, \texttt{failed})$\;
\KwRet $(\bot,\, \textsc{Tried},\, \texttt{click\_silent},\, \textsc{Suggest}(t))$\;
\caption{Three-tier click cascade with \texttt{isTrusted} preservation.}
\label{alg:cascade}
\end{algorithm}

\subsection{Motor Humanization}
\label{sec:humanize}

The motion humanizer applies four effects: B\'ezier mouse paths with Gaussian
control-point jitter, distance-proportional move timing with a sine-eased
velocity profile, optional terminal-deceleration wobble, and pre-click
hesitation drawn from $\mathcal{N}(100~\text{ms}, 50~\text{ms})$. Typing is
similarly humanized: per-character delays of $30$--$120$~ms, longer pauses
after spaces and punctuation, a 10\% chance of a $200$--$600$~ms word-pause,
and a 2\% chance of a simulated typo followed by backspace and correction.
Sliders and drag targets (chess pieces, custom range widgets) use a sigmoid
velocity profile, with a 5--15\% chance of overshoot-and-correct on long
drags. Algorithm~\ref{alg:humotion} formalises the cursor-motion generator
in both \textsc{click} and \textsc{drag} modes.

\begin{algorithm}[t]
\SetAlgoLined
\DontPrintSemicolon
\KwIn{start $\mathbf{p}_0 = (x_0, y_0)$, end $\mathbf{p}_T = (x_T, y_T)$, mode $\in \{\textsc{click}, \textsc{drag}\}$}
\KwOut{ordered waypoint sequence $\bigl((x_i, y_i, \Delta t_i)\bigr)_{i=1}^{N}$ for CDP dispatch}
\BlankLine
$d \leftarrow \|\mathbf{p}_T - \mathbf{p}_0\|_2$;\quad $N \leftarrow \max\!\bigl(\lceil d / 12 \rceil,\, 12\bigr)$\;
$T \leftarrow \textnormal{clamp}\!\bigl(d \cdot 1.8,\; 100,\; 800\bigr)\;\text{ms}$ \tcp*[r]{distance-proportional duration}
$\mathbf{p}_T \leftarrow \mathbf{p}_T + \mathcal{N}(\mathbf{0}, \sigma_T^2 \mathbf{I})$,\;
$\sigma_T = \mathbb{1}[\textsc{click}] \cdot 2 + \mathbb{1}[\textsc{drag}] \cdot 0$ \tcp*[r]{drag releases on exact target}
\BlankLine
\tcp{\textbf{Step 1.} B\'ezier control points with Gaussian perpendicular jitter}
$\mathbf{u} \leftarrow (\mathbf{p}_T - \mathbf{p}_0) / d$;\quad $\mathbf{n} \leftarrow (-u_y, u_x)$ \tcp*[r]{unit tangent and normal}
$\alpha_1 \sim \mathcal{U}(0.25, 0.40)$;\quad $\alpha_2 \sim \mathcal{U}(0.60, 0.80)$\;
$\beta_1, \beta_2 \sim \mathcal{N}(0,\, (0.12\,d)^2)$ \tcp*[r]{arc amplitude}
$\mathbf{c}_1 \leftarrow \mathbf{p}_0 + \alpha_1 (\mathbf{p}_T - \mathbf{p}_0) + \beta_1\, \mathbf{n}$\;
$\mathbf{c}_2 \leftarrow \mathbf{p}_0 + \alpha_2 (\mathbf{p}_T - \mathbf{p}_0) + \beta_2\, \mathbf{n}$\;
\BlankLine
\tcp{\textbf{Step 2.} Velocity profile: sine-easing for click, sigmoid for drag}
\eIf{$\textnormal{mode} = \textsc{click}$}{
  $\phi(t) \leftarrow \tfrac{1}{2}\bigl(1 - \cos(\pi t)\bigr)$ \tcp*[r]{slow $\to$ fast $\to$ slow}
}{
  $\phi(t) \leftarrow \sigma(8(t - 0.5))$ \tcp*[r]{logistic; emphasises deceleration}
}
\BlankLine
\tcp{\textbf{Step 3.} Emit cubic-B\'ezier waypoints with variable timing}
\For{$i \leftarrow 1$ \KwTo $N$}{
  $\tau \leftarrow \phi(i / N)$\;
  $\mathbf{p}_i \leftarrow (1-\tau)^3 \mathbf{p}_0 + 3(1-\tau)^2 \tau\, \mathbf{c}_1 + 3(1-\tau) \tau^2\, \mathbf{c}_2 + \tau^3 \mathbf{p}_T$\;
  $\Delta t_i \leftarrow \tfrac{T}{N}\bigl(1 + 0.4\,(1 - \phi'(i/N))\bigr) + \mathcal{U}(0, 5)\,\text{ms}$\;
  \lIf{$\textnormal{mode} = \textsc{drag}\;\textbf{and}\;\mathcal{U}(0,1) < 0.05$}{$\Delta t_i \mathrel{+}{=} \mathcal{U}(40, 120)\,\text{ms}$}\tcp*[r]{micro-pause}
}
\BlankLine
\tcp{\textbf{Step 4.} Terminal deceleration wobble (click only)}
\If{$\textnormal{mode} = \textsc{click}$}{
  \For{$k \leftarrow 1$ \KwTo $\textsc{Choose}(\{2, 3\})$}{
    $\mathbf{w}_k \leftarrow \mathbf{p}_T + \mathcal{N}(\mathbf{0}, 1.2^2 \mathbf{I})$;\;
    $\textnormal{emit}\;(\mathbf{w}_k,\, \mathcal{U}(8, 18)\,\text{ms})$\;
  }
}
\BlankLine
\KwRet $\bigl((x_i, y_i, \Delta t_i)\bigr)_{i=1}^{N}$\;
\caption{Humanized cursor motion: B\'ezier path with Gaussian jitter and mode-dependent velocity.}
\label{alg:humotion}
\end{algorithm}

Three worked cases tested the humanizer end-to-end. \emph{Chase IRA
calculator}: the slider is a custom \code{<mds-slider>} inside a
cross-origin iframe; the snapper detects the iframe, the humanizer drives a
sigmoid drag on the slider thumb, and the value-set helper fires a
\code{change} event on the shadow host. \emph{Coolmath4kids quiz}: inputs
live in a same-origin iframe; CDP click on the iframe host did nothing, but
\code{frame.evaluate}-driven value-set followed by \code{change} dispatch
succeeded. \emph{Chess puzzle drag}: small drag distances ($\sim$60~px),
where the sigmoid profile reduces detector false-positives.

\subsection{Direct Navigation as a Recovery Path}
\label{sec:navigate}

The click pipeline above drives a page from the inside; a complementary
recovery path sidesteps it entirely. Reaching a target \emph{state}---a search
scoped to a city, narrowed by a handful of filters---by driving the UI is a
long chain of vision-grounded clicks: open the location field, type, accept an
autocomplete, open each filter dropdown, select an option, apply. Every link in
that chain is a place a vision-driven click can dead-end, target a stale \Vn{}
index, or loop, and on filter-heavy sites it is the single most brittle stretch
of a task.

Most such sites, however, already encode that state in their URL---the location
as a path segment, the filters as query parameters. When the Worker can name
the target state as a URL, \code{browser\_navigate(session\_id, url)} moves the
session there in one deterministic hop, collapsing the entire filter chain into
a single call. The rabbit-adoption diagnostic of \S\ref{sec:modelfamilies} is a
clean example: its target---young male English Spot rabbits in Chicago---is
reachable as
\code{/search/rabbits-for-adoption/us/il/chicago/?breed=English\%20Spot\&age=Young},
and the most economical Worker we observed reached the answer with exactly one
such navigate rather than a dozen filter clicks. This is much of the mechanism
behind the tool-economy gap of \S\ref{sec:modelfamilies}: a Worker that
recognises a site's URL grammar finishes in a fraction of the calls.

\textbf{Guards.} Direct navigation is powerful enough to be dangerous, so the
tool is fenced. \emph{Domain pinning} restricts navigation to the target domain
and its subdomains plus a minimal OAuth/CDN safe-list and explicitly refuses a
pivot to a general search engine---without it, a frustrated Worker turns every
hard task into a search-engine scrape. A \emph{challenge guard} refuses
re-navigating to a URL that just returned a Cloudflare interstitial (re-issuing
the \code{goto} would only re-trigger the same challenge) and routes the Worker
to the solver first. \emph{Regression detection} flags a navigate back to an
already-visited URL, so the Worker repairs the current page instead of
restarting. A $4xx$/$5xx$ or anti-bot response short-circuits before the Worker
interacts with a dead shell, and a hard block auto-escalates the session from
the Tier-1 Puppeteer backend to the Tier-3 undetected-Chromium backend
(\S\ref{sec:cascade}). On every successful hop the vision epoch is
invalidated---the previous page's \Vn{} indices no longer apply---and a fresh
vision pass is prefetched, so the next action grounds against the page that
actually loaded.

\textbf{The double edge.} Navigation rescues cleanly only when the URL matches
the site's real structure. A Worker that fabricates a plausible-but-wrong URL
lands on an empty or mis-scoped page, and unless it notices it can report a
confident, hallucinated result---fabricated multi-parameter filter URLs are the
dominant navigate failure mode in general. The model study
(\S\ref{sec:modelfamilies}) shows the fragility even on a run where every model
ultimately succeeded: one Worker briefly navigated to a non-standard location
path (\code{/illinois/chicago/} where the site serves \code{/us/il/chicago/})
and another to a different country entirely, recovering only because the
post-navigate re-grounding surfaced the wrong page and prompted a correction.
The guards above bound the blast radius---no leaving the domain, no loops, no
interacting with a blocked shell---and the forced re-grounding catches a
mis-scoped hop, but URL \emph{fidelity} is still the Worker's responsibility,
which is why direct navigation is most reliable taken \emph{after} the page's
URL grammar is known (read off an earlier in-UI interaction) rather than guessed
cold.

\section{Experimental Evaluation}
\label{sec:eval}

We evaluate \sys{} on the Mind2Web Hard subset---66 long-horizon real-website
tasks designed to stress generalisation across domains and UI styles
\citep{deng2023mind2web}. Our primary metric is task-level success: a task
is considered successful only if every required action subsequence executed
correctly and the final outcome matches the reference. Partial credit is
\emph{not} given.

\subsection{Setup}
\label{sec:setup}

The Worker and Planner are powered by a single frontier reasoning model;
the Orchestrator uses the same model. The vision tier is served by a
small, cost-efficient multimodal model (any flash-tier vision model with
native bounding-box output suffices, e.g.\ a Gemini-3-Flash-class model
for cost efficiency; see \S\ref{sec:modelfamilies} for a discussion of
which model families work and which do not). Each task runs with a
$50$-step budget and a $5$-minute wall-clock budget. Captcha challenges are routed to a 2Captcha
Turnstile solver when detected; if the solver fails twice the task is
escalated to human handoff (such handoffs are counted as failures for the
purposes of the success metric reported here).

\subsection{Main Result}
\label{sec:mainresult}

Figure~\ref{fig:results} compares \sys{} against twenty-one published
baselines on the same 66-task subset. \sys{} attains
\textbf{$89.47\%$} task success, ranking third overall behind two
proprietary closed-API systems and ahead of every published open or
research browser-agent baseline by a wide margin.

\begin{figure}[t]
\centering
\includegraphics[width=\linewidth]{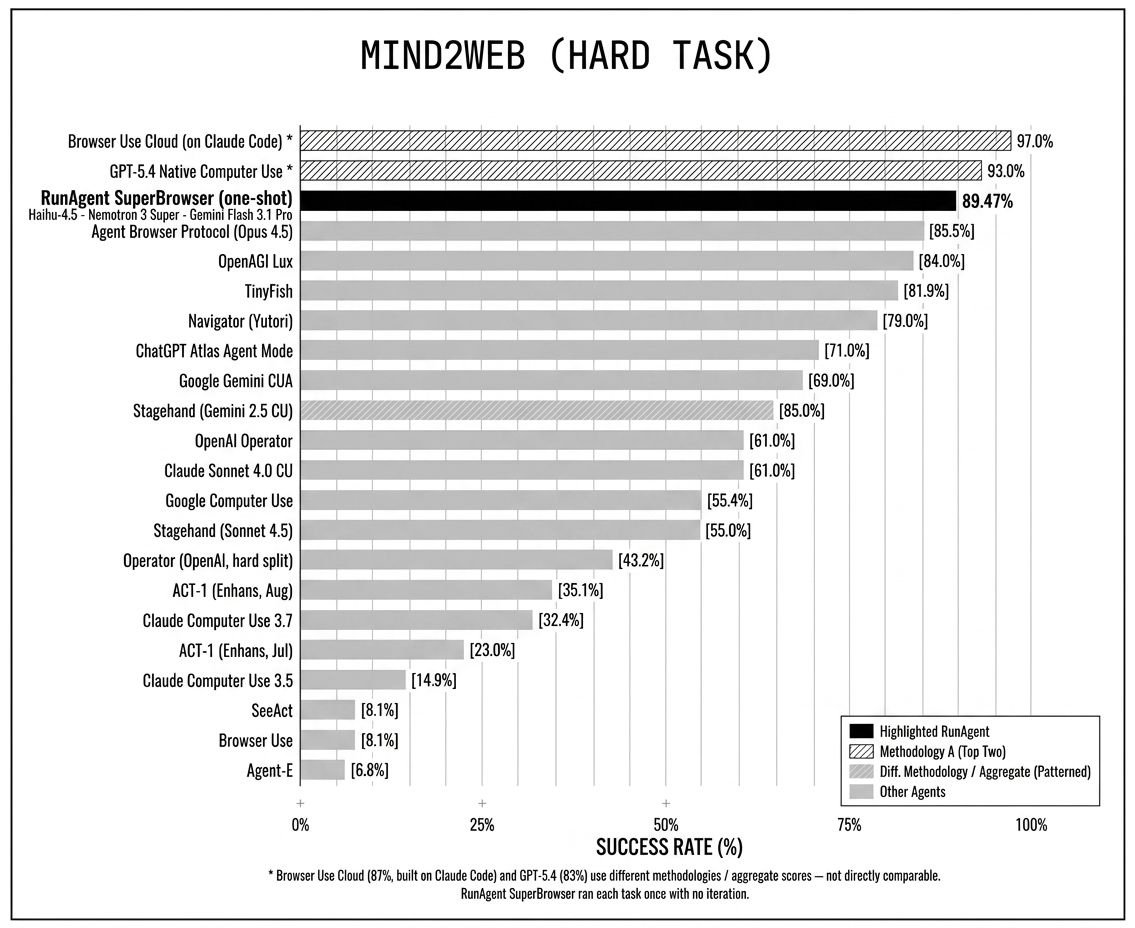}
\caption{Task success on Mind2Web Hard (66 tasks). \sys{} is highlighted in
the third position. Bars marked $\star$ denote proprietary closed-API
systems. The next-best published open / research browser agent is at
$8.1\%$, separated from \sys{} by more than eighty points.}
\label{fig:results}
\end{figure}

The most striking feature of the result is not our absolute rank but the
gap to the best published \emph{research} baseline: the next-best
non-proprietary, non-frontier-model entry is SeeAct at $8.1\%$, leaving
\sys{} more than $81$~points ahead. We attribute the gap to the joint
effect of the cognitive design contract---bounded context, structured
recall, sub-element-preferred snapping, and humanized motion---rather
than to any single component; the component analysis below is designed to
test this reading.

\subsection{Full-System Instrumentation and Ablation Plan}
\label{sec:ablations}

The cleanest test of each component is a counterfactual single-component
ablation---remove one mechanism, hold the model, vision tier, and budgets
fixed, and re-run the suite. That sweep is left to future work; here we
report the \emph{measured} behaviour of the full system on the diagnostic
task and the per-mechanism instrumentation it exposes.
Table~\ref{tab:ablation} summarises a representative full run (Claude
Opus~4.8 Worker on a web-navigation task, $41$ Worker iterations): the
task succeeds, and the cognitive memory subsystem holds the live-context
prompt bounded at ${\sim}53.5$K input tokens per iteration while firing $46$
eviction events and re-injecting the structured Ledger every turn. The
per-mechanism eviction yield is broken out in Appendix~\ref{app:metrics}
(Table~\ref{tab:evictmetrics}).

\begin{table}[t]
\caption{Full-system instrumented behaviour (Claude
Opus~4.8 Worker on a representative web-navigation task; measured over $41$
Worker iterations). A counterfactual single-component ablation sweep---removing one
mechanism at a time and re-running---is left to future work
(\S\ref{sec:ablations}); here each mechanism's \emph{measured} activity in the
full system is reported instead. Per-mechanism eviction detail is in
Appendix~\ref{app:metrics}.}
\label{tab:ablation}
\begin{center}
\small
\begin{tabular}{lr}
\toprule
\textbf{Full-system measurement} & \textbf{Value} \\
\midrule
Task success (LLM judge)               & Yes (verified listings) \\
Worker iterations                      & $41$ \\
Executed tool calls                    & $27$ \\
Live-context input tokens / iteration  & ${\sim}53.5$K \ ($46$--$59$K, bounded) \\
\midrule
Memory-eviction events fired           & $46$ \ (gut $27\,\cdot\,$failure $14\,\cdot\,$state $4\,\cdot\,$subgoal $1$) \\
Messages archived by eviction          & $70$ \\
Structured-Ledger re-injections        & $41$ \ (${\sim}1{,}080$ tokens each) \\
\bottomrule
\end{tabular}
\end{center}
\end{table}

The planned ablation sweep sets out to test the three predictions of
\S\ref{sec:predictions}. \textbf{Removing memory eviction} should let
context grow linearly with steps, increasing both tokens and hallucination
on long tasks (P1, P2)---the full-system traces already show eviction doing
the work that would otherwise let that growth happen
(Table~\ref{tab:evictmetrics}). \textbf{Removing the chevron tiebreaker}
should cause systematic failures on compound-row UIs such as filter
dropdowns and region selectors (P3). \textbf{Restricting clicks to Tier~1
only} should be brittle on bot-protected sites that block CDP click on
specific selectors; \textbf{disabling humanization} should trigger detection
on those same sites despite identical target coordinates.

\subsection{United States vs.\ Chinese Frontier Models: A Tool-Economy Gap}
\label{sec:modelfamilies}

The system contract is intentionally model-agnostic: any chat-completion
model with tool-calling support and reasonable instruction-following can
in principle drive the Worker. Holding the entire \sys{} scaffolding
fixed---same vision tier, same memory hook, same click cascade, same
prompt templates---and swapping only the Worker/Planner model across
eight recent frontier models (four from US laboratories, four from
Chinese laboratories) surfaces a consistent and, to us, surprising split.
It is not a split in \emph{whether} the task gets done, but in
\emph{how much tool work} a model spends to do it.
Figures~\ref{fig:modeltools-heatmap} and~\ref{fig:modeltools-density}
report the per-model tool economy on a representative web-navigation task
(only the Worker model varies). We present this as a single-task case
study: holding the task fixed and varying only the Worker isolates the
model effect cleanly, but the magnitude of the gap on a single task is
illustrative rather than a population estimate---a multi-task sweep that
would turn this into a distributional claim is left to a future revision.

\textbf{The economy gap.} On a fixed scaffold all eight models clear the
task: each returns a correct, verifiable listing (we hand-checked every
final answer). The outcome is a tie---success tells the two groups apart
nowhere. Where they diverge is the \emph{tool budget}. The US
Workers---Anthropic's \code{Claude-Opus-4.8} ($27$ tool calls), NVIDIA's
\code{Nemotron-3-super-120b} ($32$), and Google's \code{Gemini-3.5-Flash}
and OpenAI's \code{GPT-5.4} ($33$ each)---reach the answer in $31$ calls
on average. The recent large Chinese Workers---\code{Ring-2.6} and
MiniMax's \code{MiniMax-M3} ($36$ each), DeepSeek's \code{DeepSeek-V4}
($40$), and Moonshot's \code{Kimi-K2.6} ($53$)---average $41$, roughly a
third more tool calls ($1.3\times$) for the same result. The spread inside
the gap is what makes the point: \code{Claude-Opus-4.8} finishes in $27$
calls while \code{Kimi-K2.6} takes nearly twice as many ($53$), even
though both clear the task, and even though the Chinese models routinely
match or exceed the US models on standard reasoning, code, and math
leaderboards. Raw intelligence buys the right answer; it does not buy
\emph{tool economy}.

\begin{figure}[t]
\centering
\includegraphics[width=\linewidth]{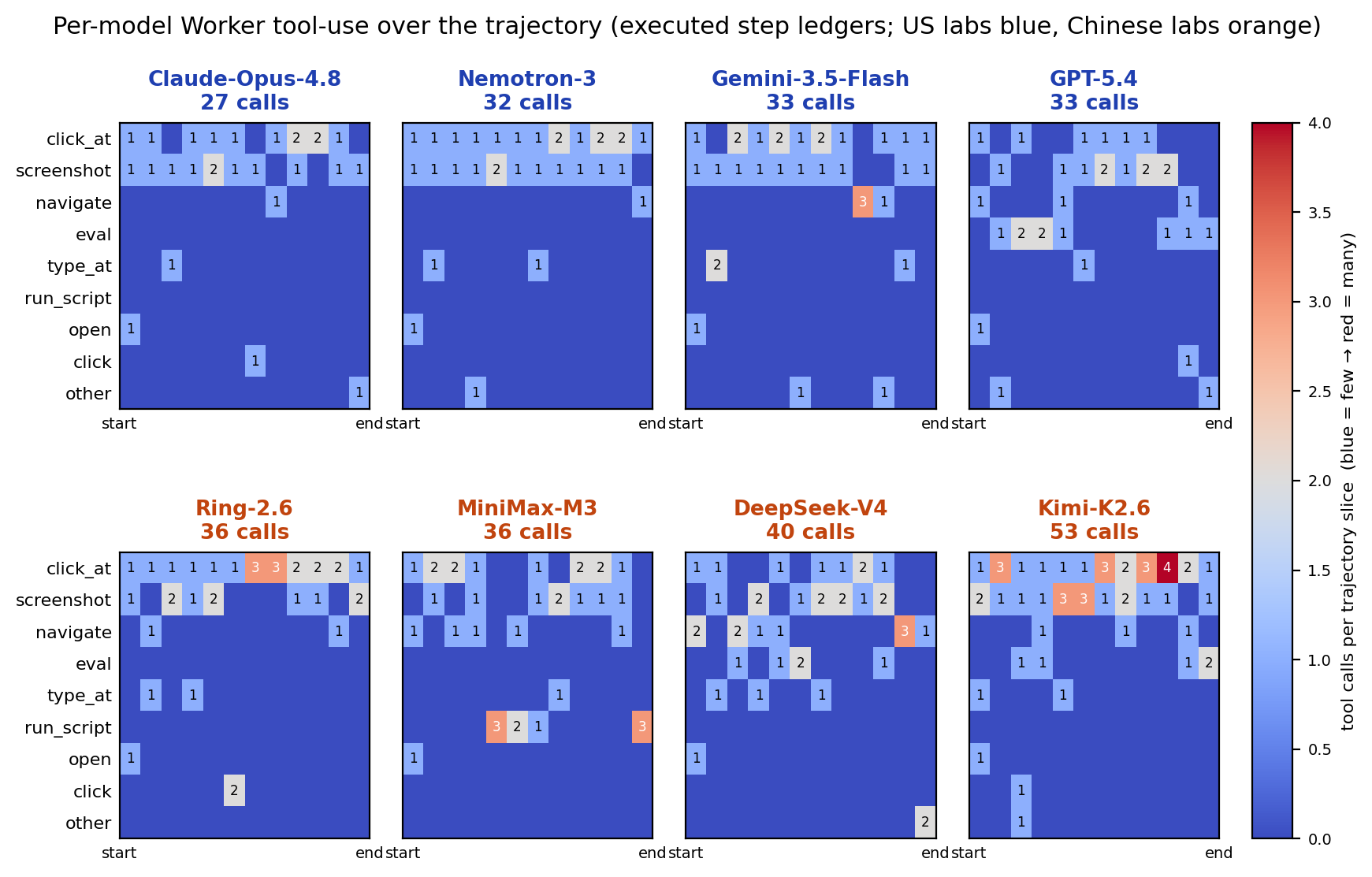}
\caption{Per-model Worker tool-use \emph{over the trajectory} on a
representative web-navigation task (only the Worker model varies), from the
executed step ledgers (\code{steps.jsonl}). One
panel per model---US labs (blue titles) on top, Chinese labs (orange) below,
ordered by total calls (shown in each title). Within a panel the $y$-axis is
the executed browser tool type and the $x$-axis is the run sliced into equal
fractions of its trajectory (start${\to}$end); each cell counts that tool's
calls in that slice, on one shared blue${\to}$red scale (blue${=}$few,
red${=}$many). US Workers front-load a short burst of vision-grounded calls
(\code{browser\_click\_at}/\code{browser\_screenshot}) and then taper, finishing
in $32$--$33$ calls; the recent large Chinese models instead sustain the vision
path across the whole trajectory---Kimi-K2.6 keeps re-issuing
\code{browser\_click\_at} (red, up to four per slice) from start to end, reaching
$53$ calls. Scaffolding, vision tier, prompts, and budgets are held fixed.}
\label{fig:modeltools-heatmap}
\end{figure}

\begin{figure}[t]
\centering
\includegraphics[width=\linewidth]{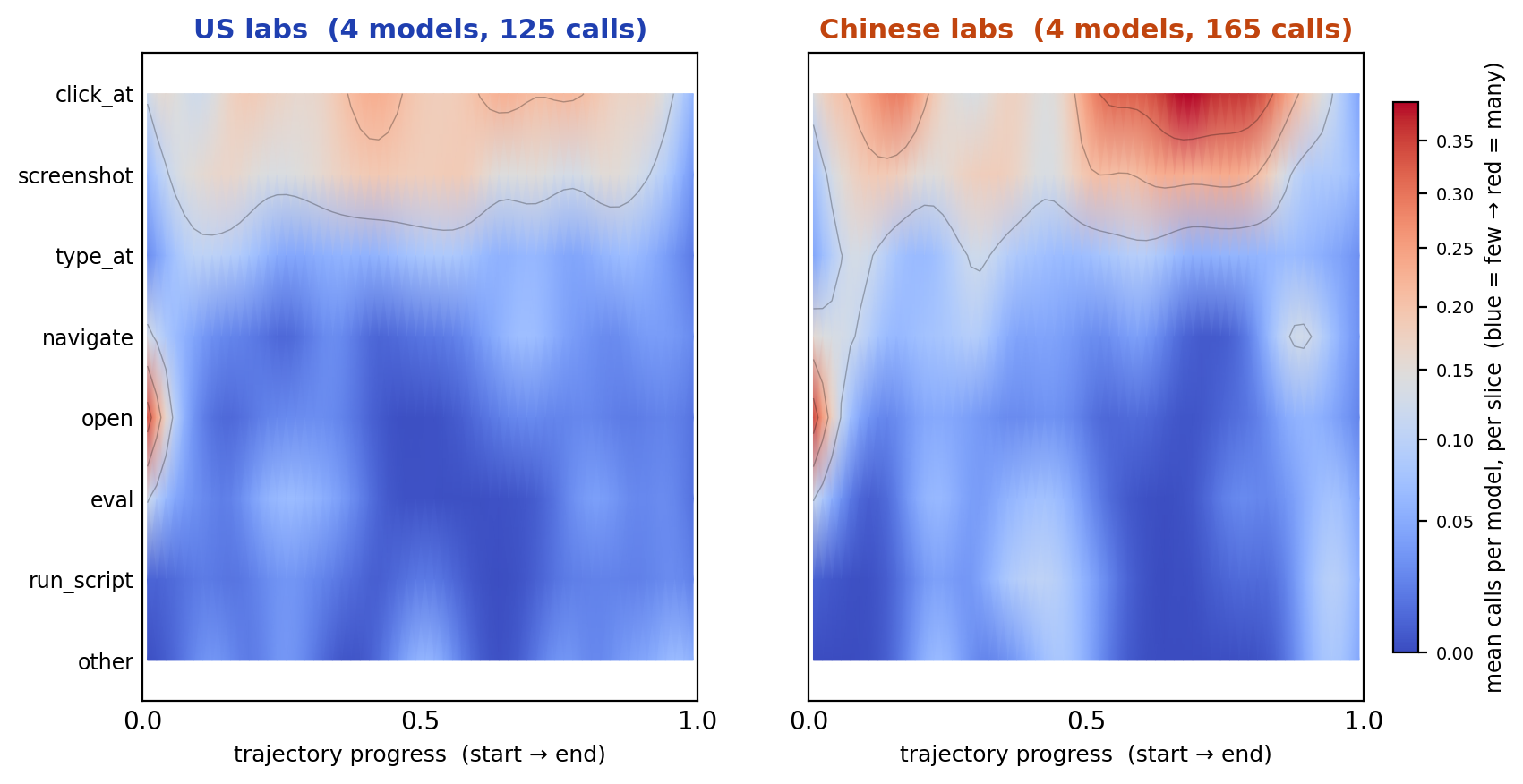}
\caption{Smoothed density of \emph{where} Worker tool calls land, pooled by lab
on a representative web-navigation task (executed step ledgers;
per-model averaged, then Gaussian-smoothed; the $y$-axis is an ordinal,
vision-grouped tool index, so this is a smoothed 2D histogram drawn as filled
contours, not a continuous-kernel density). The $x$-axis is trajectory progress
(start${\to}$end); colour is mean calls per model per slice (blue${=}$few,
red${=}$many) on a shared scale. The Chinese panel carries a markedly hotter
\code{browser\_click\_at}/\code{browser\_screenshot} band that intensifies
through the back half of the run---the redundant vision spend of
\S\ref{sec:modelfamilies}---whereas the US panel is cooler and flatter. Only
the Worker model varies.}
\label{fig:modeltools-density}
\end{figure}

\textbf{Where the extra calls go.} The surplus concentrates in the
\emph{vision-grounded} path. \code{Kimi-K2.6}, the extreme case, issues
$23$ \code{browser\_click\_at} calls and $17$ \code{browser\_screenshot}s,
re-grounding pages it has already perceived instead of acting on what it
already holds; Figures~\ref{fig:modeltools-heatmap}
and~\ref{fig:modeltools-density} show this as a hot
\code{browser\_click\_at}/\code{browser\_screenshot} band that stays lit
across the whole trajectory, where the US panels are comparatively sparse
and cool. The US Workers instead fall back to the cheap procedural path
once the page state has not changed since the last vision
call---\code{GPT-5.4}'s nine \code{browser\_eval} calls are exactly this
move---as \S\ref{sec:procedural}'s gating predicate intends. The pattern
is consistent across the Chinese Workers: each spends its surplus
re-issuing declarative, vision-grounded clicks where a cheaper procedural
action on an already-grounded element would do---precisely the redundant
vision spend the gating predicate exists to avoid.

\textbf{What it is not.} The gap is not one of raw capability: every model
succeeded on the task, and several of the Chinese models match or exceed
the US models on broad reasoning, code, and math benchmarks. It is not context length---every
candidate's context window exceeds our live-context usage by an order of
magnitude---nor multilingual capacity, since the entire task surface here
is English UI text. Our working hypothesis is \emph{post-training
distribution}: US labs appear to have invested more post-training compute
on stable agentic tool-use---function-calling fidelity over many turns,
re-using an already-grounded perception instead of re-perceiving it, and
the ``execute the registered tool, don't re-ground what you already
hold'' contract---than the recent generation of large Chinese models,
which optimise more heavily for chat, reasoning, and code-completion
benchmarks. The practical consequence is that on a fixed scaffolding
budget an ``intelligent-but-profligate'' Worker spends materially more
tool calls---and therefore more vision cost and wall-clock latency---than
a ``less-flashy-but-economical'' one to reach the same result. A natural
test of this reading---which we leave to future work---is whether
prepending a short \emph{schema-reminder} that merely restates the tool
contract (exact names, 1-based rotating \Vn{} indices, prefer the
procedural click on an already-grounded element) tightens the Chinese
models' tool economy while leaving the already-economical US models
unchanged. A prompt that adds no capability---only a restatement of the
contract---should help only if the gap is one of discipline rather than
raw ability.

\textbf{Implication for \sys.} Because the cognitive contract is the
scaffolding, not the model, this finding does not invalidate the
architecture; if anything it strengthens the argument that
\emph{discipline}---of context, of tool calls, of perception---is what
makes long-horizon tasks cheap and stable, not merely solvable. The
corollary for users is practical: when selecting a Worker model, agentic
tool-use \emph{economy} should weigh as heavily as headline reasoning
scores, since on a fixed scaffold the two diverge by a third or more in
tool cost at equal success.

\subsection{Qualitative Traces}
\label{sec:qualitative}

We highlight three failure modes that the cognitive design specifically
addresses; full trace dumps are in Appendix~\ref{app:traces}.

\textbf{Region selector (chevron tiebreaker).} On a flight-search site, the
``Departure airport'' filter row contains the text ``United States'' and a
narrow ``$\blacktriangledown$'' chevron. The first vision pass returns a
single bounding box spanning the entire row. Without the chevron
tiebreaker, the snapper clicks the text label centre and nothing happens.
With the tiebreaker, the snapper scores the chevron candidate at $\sigma =
\text{area}_{\text{chev}} \times 3.0$ versus $\sigma_{\text{label}} =
\text{area}_{\text{label}} \times 1.0$ and prefers the chevron; the
dropdown opens.

\textbf{Cross-origin slider (Chase IRA).} The Chase retirement calculator
hosts its slider widget inside a cross-origin iframe loaded from
\code{static.chasecdn.com}; the slider is a Lit web component
(\code{<mds-slider>}) with hidden \code{<input type="range">} inside an
open shadow root. The snapper detects the iframe at Phase A-alt; the
humanizer dispatches a sigmoid-velocity drag on the thumb; a shadow-DOM
helper fires a \code{change} event on the host. Without humanization the
slider validates as a bot; without shadow-DOM piercing the value never
propagates.

\textbf{Iframe value-set (coolmath4kids).} The quiz inputs sit inside a
same-origin iframe; CDP click on the iframe host did not focus the inner
input. We route through \code{frame.evaluate} + native value setter,
followed by a \code{change} dispatch, which is more robust than synthesised
clicks here.

\section{Discussion and Limitations}
\label{sec:discussion}

\textbf{Compute cost.} Each task averages $11$ vision calls and $\sim$22
LLM calls under our $50$-step budget. The asynchronous vision prefetch
hides the vision call behind the next LLM call, so wall-clock cost grows
with LLM latency, not vision latency. We have not measured FLOPs or
dollars in this paper; the framework is independent of the underlying
LLM, so users can trade quality for cost by swapping models. As noted in
\S\ref{sec:modelfamilies}, the practical model-selection axis is
agentic tool-use fidelity rather than raw reasoning capacity, so the
cheapest disciplined Worker often outperforms a more expensive but
less-disciplined one.

\textbf{Dependence on a vision API.} Cold-start tasks---a page on which
the vision model returns no bounding box above the confidence
threshold---force the Worker to fall back to DOM enumeration. This is
slow and brittle. In our internal traces the cold-start fallback rate is
about $4\%$ of pages; on those pages the success rate drops by roughly
$15$ points relative to the vision-available case. A vision model with
better recall on minimal-UI or non-Latin-script pages would help most.

\textbf{Cognitive analogy as motivation, not claim.} We are not claiming
that \sys{} is in any way conscious or that LLMs ``think''. The
contribution is the architecture: \emph{the bounded-context structured
ledger with role-sliced views and six-phase eviction works
experimentally}, whether or not the underlying model is in any sense
cognitive. The analogy is a generative source for the design and a
mnemonic for the reader; it is not a load-bearing claim of equivalence.

\textbf{Reproducibility.} The system is implemented in TypeScript
(browser layer) and Python (Nanobot bridge and orchestration). The
Mind2Web Hard evaluation was run on commit \code{cbfed0d} of
the public repository. We will release the eval harness alongside the
camera-ready paper.

\textbf{Open questions.} Three directions seem most promising. First,
the Ledger schema is currently hand-designed; learning what to remember
on a per-domain basis is an open problem. Second, the chevron
tiebreaker generalises to compound rows but not to all sub-element
ambiguities (e.g.\ nested cards); a more general saliency-based snapper
is desirable. Third, the cognitive analogy makes a fourth prediction we
have not tested: that interrupting and resuming a task should be near
free if the Ledger is the only state. We leave that experiment for
future work.

\section{Conclusion}
\label{sec:conclusion}

We presented \sys, a web-navigation agent designed against an explicit
cognitive contract: perception is candidate generation, cognition is
split between a slow strategic loop and a fast operational loop, memory
is a bounded structured ledger from which observations are systematically
evicted, and action is a humanized cascade that prefers small interactive
sub-elements over large text labels. On Mind2Web Hard, \sys{} attains
$89.47\%$ task success and outpaces every published open or research
browser-agent baseline by more than $80$ points. The result suggests that
context \emph{discipline}---deciding what \emph{not} to remember---is at
least as important as context \emph{capacity}, and that web agents
designed with a working-memory analogy in mind degrade more gracefully
on long-horizon tasks than agents that simply accumulate.

The system and evaluation harness will be open-sourced. We hope the
structured-Ledger interface and the six-phase eviction protocol are
useful as a drop-in building block for other agentic systems.

\bibliography{bibliography}
\bibliographystyle{iclr2026_conference}

\appendix
\section{Complete Browser-Tool Taxonomy}
\label{app:tools}

The Worker has access to 36 browser tools, organised into nine functional
categories. The two most important for the body of the paper are
\code{browser\_click\_at} and \code{browser\_type\_at}, which route through
the bounding-box snapper and the text-fix pre-pass respectively.

\begin{table}[h]
\caption{All 36 browser tools available to the Worker, by category.}
\label{tab:tools}
\begin{center}
\small
\begin{tabularx}{\linewidth}{lXl}
\toprule
\textbf{Category} & \textbf{Tool} & \textbf{Layer} \\
\midrule
Navigation     & \code{navigate}, \code{search\_google}, \code{go\_back} & Motor \\
\midrule
Interaction    & \code{click\_element}, \code{input\_text}, \code{select\_option}, & Motor \\
               & \code{send\_keys}, \code{get\_dropdown\_options}, \code{select\_dropdown\_by\_text} & \\
\midrule
Scroll         & \code{scroll\_down}, \code{scroll\_up}, \code{scroll\_to\_percent}, & Motor \\
               & \code{scroll\_to\_top}, \code{scroll\_to\_bottom}, \code{scroll\_to\_text}, & \\
               & \code{scroll\_pixels}, \code{scroll\_within}              & \\
\midrule
Tabs           & \code{open\_tab}, \code{switch\_tab}, \code{close\_tab}    & Motor \\
\midrule
Control        & \code{done}, \code{wait}                                  & Cognition \\
\midrule
Advanced       & \code{handle\_dialog}, \code{upload\_file},               & Motor \\
               & \code{evaluate\_script}, \code{run\_script}              & \\
\midrule
Extraction     & \code{extract\_markdown}, \code{export\_pdf},             & Perception \\
               & \code{dom\_search}, \code{wait\_for\_condition},          & \\
               & \code{get\_console\_errors}, \code{get\_accessibility\_tree} & \\
\midrule
Captcha        & \code{detect\_captcha}, \code{screenshot\_captcha},       & Mixed \\
               & \code{solve\_captcha\_visual}                             & \\
\midrule
Geo            & \code{detect\_geo\_block}                                 & Perception \\
\bottomrule
\end{tabularx}
\end{center}
\end{table}

Each tool returns a structured \code{ToolResult\{success, reason, tried[],
alternatives, error\}}. The \code{tried} list records which click-cascade
tiers were attempted (e.g.\ \code{['cdp','puppeteer']}), and
\code{alternatives} provides the LLM with concrete next moves so it does not
need to re-screenshot just to learn what to try next. This contract holds
across all interactive tools.

\section{Eviction Metrics}
\label{app:metrics}

Table~\ref{tab:evictmetrics} reports the memory subsystem's per-mechanism
eviction activity, measured directly from the Worker event ledger on a
representative run (Claude Opus 4.8 Worker on a web-navigation task, $41$
Worker iterations). Counts are exact---each row is an
event type the memory hook emits when it fires---so the table records
\emph{how often} each mechanism acts and \emph{how much} it removes or
re-injects per firing. Per-phase \emph{token} attribution is not separately
instrumented; what is instrumented, and reported here, is the firing count,
the items removed per firing, and the net effect on live-context size.

\begin{table}[h]
\caption{Per-mechanism eviction activity over $41$ Worker iterations of a
representative run (Claude Opus 4.8 on a web-navigation task). ``Removed /
firing'' is in the unit named in the last column. The structured Ledger
re-injects the compacted authoritative state every turn, which is what makes
the aggressive gutting safe.}
\label{tab:evictmetrics}
\begin{center}
\small
\begin{tabular}{lccl}
\toprule
\textbf{Mechanism} & \textbf{Firings} & \textbf{Removed / firing} & \textbf{What it strips / adds} \\
\midrule
Old-content gut         & $27$ & $2.6$ messages & turns past the window $\to$ short markers \\
Failure collapse        & $14$ & $1.7$ failures & older failures $\to$ one-line cause \\
State-block eviction    & $4$  & $1$ block      & older \code{[SESSION\_STATE]} blocks stripped \\
Subgoal compaction      & $1$  & $83$ messages  & finished subgoal $\to$ one debrief line \\
\midrule
Ledger re-injection     & $41$ & $\sim\!1{,}080$ tokens & compact authoritative state re-sent each turn \\
\bottomrule
\end{tabular}
\end{center}
\end{table}

The net effect is a bounded live prompt: total input tokens per iteration stay
in the $46$--$59$K band (mean $\sim\!53.7$K) across the run rather than growing
with step count---$46$ gut/collapse/compaction events fire over the $41$
iterations to hold that envelope while the Ledger (re-injected every turn at
$\sim\!1{,}080$ tokens) preserves the state those evictions would otherwise
lose. (Prompt-caching ratios are provider-dependent; this run was served
without cache reuse, so we report raw input tokens.)

\section{Qualitative Trace Excerpts}
\label{app:traces}

We include short trace fragments from the three worked cases in
\S\ref{sec:qualitative}. Bracketed numbers reference the Worker step in the
underlying \code{steps.jsonl}.

\subsection{Region selector chevron (flight search)}

\begin{quote}\small\itshape
[12] vision: \Vidx{4} \emph{``United States $\blacktriangledown$''} bbox=(304,212,612,244).
[12] worker $\to$ \code{browser\_click\_at(vision\_index=4, target\_label="United States")}.
[12] snapper: phase A pinpoint $\to$ static text node, not interactive. Phase B
grid scan: candidate (a)~\code{<button aria-expanded="false">} score
$\sigma=\text{area}_{a}\cdot 3.0$; candidate (b)~\code{<span>United States</span>}
score $\sigma=\text{area}_{b}\cdot 1.0$. Chevron tiebreaker selects (a).
[12] cdp click at (584,228) $\to$ dropdown opens, 51 regions revealed.
\end{quote}

\subsection{Cross-origin slider (Chase IRA calculator)}

\begin{quote}\small\itshape
[17] vision: \Vidx{9} \emph{``Salary slider thumb''} bbox=(421,580,468,612).
[17] snapper: phase A $\to$ \code{<iframe src="static.chasecdn.com/...">};
Phase A-alt iframe descent: \code{contentDocument} accessible; inner pinpoint
$\to$ \code{<mds-slider>} custom element with hidden \code{<input type="range">}
in open shadow root. Strategy: shadow-DOM piercing via \code{\_\_sb\_queryDeep}.
[17] humanizer: sigmoid drag from (442,596) to (608,596), $\sim$520~ms total,
$5$~micro-pauses, $1$~overshoot+correction at $14\%$.
[17] post-action: native \code{change} event fired on shadow host; recomputed
display reads ``\$110,000''. Constraint ``salary $\ge$ \$100k'' satisfied.
\end{quote}

\subsection{Same-origin iframe value-set (coolmath4kids)}

\begin{quote}\small\itshape
[8] worker $\to$ \code{browser\_type\_at(vision\_index=3, text="42")}.
[8] snapper: \Vidx{3} resolves to \code{<input type="number">} inside same-origin
iframe \code{<iframe src="/quiz/embed.html">}. CDP type on iframe host loses focus.
[8] fallback: route through \code{frame.evaluate}: locate input, call
\code{nativeInputValueSetter.call(input, "42")}, dispatch \code{InputEvent}
and \code{ChangeEvent}. Validator accepts; ``Submit'' button enables.
\end{quote}

\subsection{Worker-slot failure traces (Chinese models)}
\label{app:modelsplit-traces}
Representative Worker tool calls from Chinese-model runs and the system's
response, illustrating the three failure modes of \S\ref{sec:modelfamilies}.
Auto-generated by \code{eval/figures/make\_appendix\_traces.py}.


\emph{(Trace listings will appear here once Chinese-model eval runs are available.)}

\section{Verb-Classifier Routing Algorithm}
\label{app:routing}

The full pseudocode of the Orchestrator's verb classifier
(\S\ref{sec:routing}) is given in Algorithm~\ref{alg:routing}.

\begin{algorithm}[t]
\SetAlgoLined
\DontPrintSemicolon
\KwIn{instructions $I$, optional URL $u$}
\KwOut{$(\text{kind} \in \{\textsc{browser}, \textsc{search}\}, \text{confidence})$}
\BlankLine
$s_{\text{br}} \leftarrow 0$;\quad $s_{\text{se}} \leftarrow 0$\;
\lIf{$I$ matches \textsc{ActionVerbs}}{$s_{\text{br}} \mathrel{+}= 0.80$}
\lIf{$I$ matches \textsc{TransactionalPatterns} $\wedge$ \textsc{DateIndicators}}{$s_{\text{br}} \mathrel{+}= 0.90$}
\lIf{$I$ matches \textsc{ExplicitTargetVerbs}}{$s_{\text{br}} \mathrel{+}= 0.85$}
\lIf{$I$ matches \textsc{VisualOnly}}{$s_{\text{br}} \mathrel{+}= 0.80$}
\lIf{$I$ matches \textsc{BrandNames}}{$s_{\text{br}} \mathrel{+}= 0.70$}
\lIf{$I$ matches \textsc{AggregationVerbs} $\wedge$ \textsc{PluralContext}}{$s_{\text{se}} \mathrel{+}= 0.70$}
\lIf{$I$ matches \textsc{FactualLookup} $\wedge \neg\textsc{VisualOnly}$}{$s_{\text{se}} \mathrel{+}= 0.60$}
\BlankLine
\If{$u \ne \emptyset$}{
  $\ell \leftarrow \textsc{LoadDomainLearnings}(\textsc{Domain}(u))$\;
  \If{$\ell.\text{preferred}$ \textnormal{exists} $\wedge\; \ell.\text{confidence} > 0.8$}{
    \KwRet $(\ell.\text{preferred},\; \ell.\text{confidence})$\;
  }
}
$\text{kind} \leftarrow \textbf{if } s_{\text{br}} \ge s_{\text{se}} \textbf{ then } \textsc{browser} \textbf{ else } \textsc{search}$\;
\KwRet $(\text{kind},\; \max(s_{\text{br}}, s_{\text{se}}))$\;
\caption{Verb-classifier task routing.}
\label{alg:routing}
\end{algorithm}

In our implementation, the pattern families are concrete regular expressions
maintained in the orchestrator's source file. The classifier is deterministic
and contributes negligible latency relative to LLM inference; its main role
is to refuse to route lookup queries to the (expensive) browser worker.

\section{Mathematical Formalism}
\label{app:formal}

This appendix collects formal definitions for the mechanisms described
informally in the main text: the cognitive memory operator
(\S\ref{app:eviction-formal}), the bounding-box snap scoring function
(\S\ref{app:snap-formal}), the three-role loop expressed as a constrained
Markov decision process (\S\ref{app:mdp-formal}), a lower bound on prompt-cache
hit rate (\S\ref{app:cache-formal}), and an information-foraging derivation of
the verb-classifier router (\S\ref{app:scent-formal}). Nothing here is needed
to read the main paper; the goal is to make the system's invariants and
optimisation criteria precise enough to be falsified or improved upon.

\subsection{Notation and preliminaries}
\label{app:notation}

Let $\mathcal{T}$ be the set of user tasks, $\mathcal{S}$ the set of page
states (a DOM tree, a viewport screenshot, and a console transcript), and
$\mathcal{A}$ the action space (the $36$ browser tools in
Appendix~\ref{app:tools}). A run of horizon $H$ produces a trajectory
$\tau = (s_0, a_0, s_1, a_1, \ldots, s_H) \in (\mathcal{S} \times \mathcal{A})^{H}$.

Let $\mathcal{M}$ be the space of LLM message lists (sequences of typed
messages with roles in $\{\textsf{system}, \textsf{user}, \textsf{assistant},
\textsf{tool}\}$). Let $T : \mathcal{M} \to \mathbb{N}$ denote token count.
Let $\Sigma$ be the space of vision bounding boxes:
\begin{equation}
  \Sigma \;=\; \bigl\{\, (x_0, y_0, x_1, y_1, \ell, c, \mathbf{m}) \,:\, x_0 < x_1,\, y_0 < y_1,\, c \in [0,1] \,\bigr\},
\end{equation}
where $\ell$ is a short label, $c$ is the vision model's confidence, and
$\mathbf{m}$ is the DOM-enriched metadata vector containing the matching
DOM index, \texttt{aria-expanded} state, active-state flag, resolved
\texttt{aria-labelledby} text, and the parent expand-control index
(\S\ref{sec:vision}).

The Ledger space is a Cartesian product of typed slots:
\begin{equation}
\mathcal{L}
\;=\;
\underbrace{\Sigma^*_{\text{goal}}}_{\text{goal}}
\times
\underbrace{(\Sigma^*)^*}_{\text{plan}}
\times
\underbrace{\mathcal{O}_{\le 3}}_{\text{recent}_3}
\times
\underbrace{\mathcal{F}}_{\text{facts}}
\times
\underbrace{\mathcal{D}^*}_{\text{deadEnds}}
\times
\underbrace{\mathcal{C}^*}_{\text{checkpoints}}
\times
\underbrace{(\Sigma^*)^*}_{\text{episodic}}\,,
\end{equation}
where $\mathcal{O}_{\le k}$ denotes the type of bounded deques of length
at most $k$ over step-outcomes, $\mathcal{F}$ is a finite map from string
keys to facts with category, confidence, and most-recently-referenced
timestamp, and $\mathcal{D}, \mathcal{C}$ are types for dead-end and
checkpoint records as in \S\ref{sec:ledger}.

\subsection{The cognitive eviction operator}
\label{app:eviction-formal}

Each of the six eviction phases (\S\ref{sec:eviction}) is a function
$E_i : \mathcal{M} \times \mathcal{L} \to \mathcal{M} \times \mathcal{L}$.
The composite cognitive eviction operator is
\begin{equation}
E \;=\; E_6 \circ E_5 \circ E_4 \circ E_3 \circ E_2 \circ E_1.
\label{eq:E-composite}
\end{equation}

\paragraph{Idempotence.} Each $E_i$ is idempotent: $E_i \circ E_i = E_i$.
Phases $E_1, E_3, E_4, E_6$ act on disjoint message kinds (image blocks,
state blocks, thinking blocks, element-list blocks respectively) and
therefore pairwise commute. Phase $E_2$ writes to the ledger but does not
re-read evicted material. Phase $E_5$ subsumes the others on its acted-on
turns. Hence
\begin{equation}
E \circ E \;=\; E,
\label{eq:E-idempotent}
\end{equation}
which is what lets the hook run safely on every iteration.

\paragraph{Token-bound invariant.} Let $\bar{m}_{\text{step}}$ be the
average per-step message size (state block + tool result + vision caption +
worker response). After $t$ steps,
\begin{equation}
T(M_t) \;\le\;
\underbrace{T(M_{\text{sys}})}_{\text{stable prefix}}
\;+\;
\underbrace{k \cdot \bar{m}_{\text{step}}}_{\text{live window}}
\;+\;
\underbrace{O(\log t)}_{\text{archived markers}},
\label{eq:token-bound}
\end{equation}
with $k = 5$ in our deployment (Phase 5 keeps the last $5$ turns). The
naive control loop has $T(M_t) = T(M_{\text{sys}}) + t \cdot \bar{m}_{\text{step}}$,
\emph{linear} in $t$. The empirical realisation of
\eqref{eq:token-bound} is the right-hand panel of
Figure~\ref{fig:memcompare} in \S\ref{sec:memempirics}.

\paragraph{Ledger conservation.} Define semantic content of the Ledger as
the disjoint union of its typed slots:
$\|\mathcal{L}\|_\text{sem} := |\text{facts}| + |\text{deadEnds}| +
|\text{checkpoints}| + |\text{episodic}| + |\text{all\_steps}|$.
Eviction never removes from this measure:
\begin{equation}
  \|\,(E(M, \mathcal{L}))_{\mathcal{L}}\,\|_\text{sem}
  \;\ge\;
  \|\mathcal{L}\|_\text{sem}.
\label{eq:ledger-monotone}
\end{equation}
Each phase either leaves the ledger unchanged (Phases 1, 3, 4, 5, 6) or
\emph{adds} entries to it (Phase 2 emits a DeadEnd whenever it collapses
a failure). Thus the live window shrinks while the ledger grows
monotonically.

\subsection{Bounding-box snap as a scoring function}
\label{app:snap-formal}

Given a vision bounding box $B \in \Sigma$ with optional expected label
$\ell^\star$, and the set of candidate interactive elements
$\{e_1, \ldots, e_n\}$ with rectangles $R_i$ and attribute vectors
$\boldsymbol\alpha_i$, the snapper returns
\begin{equation}
e^\star \;=\; \arg\max_{i \in [n]} \,\sigma(e_i, B, \ell^\star),
\label{eq:argmax-snap}
\end{equation}
under the composite score
\begin{equation}
\sigma(e, B, \ell^\star)
\;=\;
a(R, B) \cdot \lambda(e, \ell^\star) \cdot \chi(e, B),
\label{eq:sigma-snap}
\end{equation}
with the three factors defined as follows.

\paragraph{Area factor.} The viewport-area intersection ratio
\begin{equation}
a(R, B) \;=\;
\frac{\text{area}(R \cap B)}{\text{area}(B)} \in [0, 1].
\end{equation}

\paragraph{Label factor.} A type-dependent label match,
\begin{equation}
\lambda(e, \ell^\star)
\;=\;
\begin{cases}
\mathbb{1}\!\left[\textsc{Tokens}(\ell^\star) \subseteq \textsc{TextTokens}(e)\right] & \text{(calendar day)}\\[2pt]
\textsc{Substr}(\ell^\star,\, e.\text{text}) & \text{(dropdown / option)} \\[2pt]
\textsc{JaccardWordOverlap}\!\big(\ell^\star, e.\text{text}\cup e.\text{aria}\big) & \text{(value-bearing trigger)} \\[2pt]
0.5 + 0.5\,\mathbb{1}\!\left[\ell^\star \approx e.\text{text}\right] & \text{(generic)}
\end{cases}
\end{equation}
where the first three rules implement the per-widget leniencies described
in \S\ref{sec:bbox} (calendar day cells require exact word match; dropdown
items allow substring; value-bearing triggers such as time pickers grant
$0.7$ on a $\ge 3$-character word overlap).

\paragraph{Chevron factor.} Let $K(e, B)$ be the number of chevron
indicators on $e$ within the row-shaped bounding box $B$:
\begin{equation}
K(e, B) \;=\;
\mathbb{1}[\text{aria-expanded} \in e]
+ \mathbb{1}[\text{aria-haspopup} \in e]
+ \mathbb{1}[\textsc{Glyph}(e) \in \{\blacktriangledown,\blacktriangleright,\blacktriangleleft,\blacktriangle\}]
+ \mathbb{1}[\text{aria-label}(e) \ni \{\text{``expand''}, \text{``collapse''}, \text{``toggle''}, \text{``more''}\}],
\end{equation}
applied only when $B$ satisfies the row predicate
$\textsc{Row}(B) = \mathbb{1}[\text{width}(B) \ge 60\,\text{px}] \cdot \mathbb{1}[\text{height}(B) \ge 24\,\text{px}]$.
Then
\begin{equation}
\chi(e, B) \;=\;
1 \,+\, 2\,\textsc{Row}(B) \cdot \mathbb{1}[K(e, B) \ge 1] \cdot \tanh\!\bigl(K(e, B) / 2\bigr),
\label{eq:chi}
\end{equation}
which saturates near $\chi = 3$ for $K \ge 3$ and equals $1$ for non-row
boxes or non-chevron elements.

\begin{proposition}[Chevron tiebreaker]
\label{prop:chev}
Let $e_L$ be a text-label candidate with $\chi(e_L) = 1$ and $e_C$ a
chevron candidate with $\chi(e_C) \ge 2$, both inside a row-shaped bbox
$B$. If their label scores satisfy
$\lambda(e_C, \ell^\star) \ge \lambda(e_L, \ell^\star)$, then
$\sigma(e_C, B, \ell^\star) > \sigma(e_L, B, \ell^\star)$ whenever
\begin{equation}
\frac{a(R_L, B)}{a(R_C, B)} \;<\; \chi(e_C).
\label{eq:chev-condition}
\end{equation}
\end{proposition}

\begin{proof}
Direct substitution into \eqref{eq:sigma-snap}. The label factors cancel
or favour $e_C$ by assumption, so
$\sigma(e_C) / \sigma(e_L) \ge \chi(e_C) \cdot a(R_C, B) / a(R_L, B)$,
which exceeds $1$ exactly when \eqref{eq:chev-condition} holds. \qed
\end{proof}

\paragraph{Interpretation.} The proposition formalises the
``United~States~$\blacktriangledown$'' scenario of \S\ref{sec:chevron}. As
long as the chevron sub-element occupies at least $\frac{1}{\chi(e_C)}$
of the label's area, the chevron wins. For $\chi(e_C) = 3$ this means
the chevron need only be $\sim 33\%$ as large as the label to be
preferred---a generous margin for real UIs where chevrons are typically
$10$--$20\%$ of the label's footprint.

\subsection{Three-role loop as a constrained MDP}
\label{app:mdp-formal}

Formally, \sys{} solves a partially observed Markov decision process
\begin{equation}
\mathcal{P} \;=\; (\mathcal{S}, \mathcal{A}, P, R, \gamma)
\end{equation}
with state $s_t = (\text{page}_t, \mathcal{L}_t)$, transition kernel $P$
induced by the browser, reward $R$ given by a task-success indicator at
horizon $H$, and discount $\gamma = 1$ (finite horizon). The novelty is
that the policy is not a single function $\pi(a \mid s)$ but the
composition of three role-specific policies
$\pi_O, \pi_P, \pi_W$ operating on \emph{role-sliced views} of the state:
\begin{equation}
v_W(s) = \bigl(\text{page}_t,\; v_W(\mathcal{L}_t)\bigr),\quad
v_P(s) = \bigl(\text{page}_t,\; v_P(\mathcal{L}_t)\bigr),\quad
v_O(s) = \bigl(-, v_O(\mathcal{L}_t)\bigr),
\end{equation}
where the ledger projections are:
\begin{align}
v_W(\mathcal{L}) &= \bigl(\text{goal},\, \text{subgoal},\, \text{recent}_3,\, \text{top-}5(\text{facts}),\, \text{deadEnds}_{\text{url}=u_t}\bigr) \\
v_P(\mathcal{L}) &= \bigl(\text{goal},\, \text{plan},\, \text{subgoal},\, \text{recent}_3,\, \text{facts},\, \text{deadEnds},\, \text{checkpoints}\bigr) \\
v_O(\mathcal{L}) &= \mathcal{L} \;\oplus\; \text{learnings}_{\text{domain}(u_t)}.
\end{align}

The orchestrator runs once per task; the planner runs every $N$ steps;
the worker runs every step:
\begin{equation}
\pi(a_t \mid s_t) \;=\;
\begin{cases}
\pi_O\bigl(\text{kind} \mid v_O(s_t)\bigr) & t = 0,\\[2pt]
\pi_P\bigl(\text{plan}_t \mid v_P(s_t)\bigr) \cdot \pi_W\bigl(a_t \mid v_W(s_t), \text{plan}_t\bigr) & t \bmod N = 0,\\[2pt]
\pi_W\bigl(a_t \mid v_W(s_t), \text{plan}_{\lfloor t/N \rfloor \cdot N}\bigr) & \text{otherwise}.
\end{cases}
\end{equation}

\begin{remark}[Role separation as attentional gating]
The role-sliced projections $v_W, v_P, v_O$ are precisely the
attentional-gating mechanism predicted by Baddeley's central-executive
account \citep{baddeley2000episodic}: the Worker (operational layer) is
shielded from strategic state it cannot act on, just as motor systems
in humans are shielded from non-motor representations.
\end{remark}

\subsection{Cache-hit lower bound}
\label{app:cache-formal}

Let $\rho_t \in [0, 1]$ denote the prompt-cache hit rate at iteration
$t$. Modern LLM caches index a sliding-window suffix of stable tokens;
let $P_t$ be the byte-identical prefix of $M_t$ that has survived since
$M_{t-1}$. Then
\begin{equation}
\rho_t \;\ge\; \frac{T(P_t)}{T(M_t)}
\;\ge\; 1 \,-\, \frac{T(\Delta_t)}{T(M_t)},
\end{equation}
where $\Delta_t$ is the diff (insertions, edits, deletions) between
$M_{t-1}$ and $M_t$. With cognitive eviction, $\Delta_t$ has three
components: (i) the appended worker turn $\delta_a$, (ii) the appended
state and tool-result blocks $\delta_s$, and (iii) the rendered ledger
delta $\delta_\mathcal{L}$. The first two are $O(\bar{m}_{\text{step}})$.
The third is small because the ledger renders most slots verbatim and
only the per-step \texttt{recent} deque rotates; in our deployment
$T(\delta_\mathcal{L}) \le 250$ tokens.

Combining with \eqref{eq:token-bound},
\begin{equation}
\rho_t \;\ge\; 1 \,-\, \frac{\bar{m}_{\text{step}} + T(\delta_\mathcal{L})}{T(M_{\text{sys}}) + k \bar{m}_{\text{step}} + O(\log t)}
\;\xrightarrow[t \to \infty]{}\;
1 \,-\, \frac{\bar{m}_{\text{step}} + T(\delta_\mathcal{L})}{T(M_{\text{sys}}) + k\bar{m}_{\text{step}}}.
\end{equation}
With $T(M_{\text{sys}}) \approx 2000$, $\bar{m}_{\text{step}} \approx 600$,
$k = 5$, $T(\delta_\mathcal{L}) \le 250$, this gives
$\rho_\infty \gtrsim 0.83$. The empirical mean of $\sim 0.87$ in
\S\ref{sec:memempirics} is consistent with this bound.

\subsection{Information scent as expected utility}
\label{app:scent-formal}

The verb-classifier router of \S\ref{sec:routing} can be derived as an
approximation to the information-scent decision rule of
\citet{pirolli2007foraging}. Let $\hat{u}(\text{kind} \mid I)$ denote the
forager's expected utility of choosing $\text{kind} \in
\{\textsc{browser}, \textsc{search}\}$ for instructions $I$, with utility
defined as success probability scaled by negative wall-clock cost.
Foraging optimality requires
\begin{equation}
\text{kind}^\star = \arg\max_{\text{kind}} \hat{u}(\text{kind} \mid I).
\end{equation}

Modeling $\hat{u}$ as a log-linear combination of binary scent features
$\phi_f(I) = \mathbb{1}[I \text{ matches family } f]$ over the feature
families $\mathcal{F} = \{\text{action},\text{aggregation},\text{factual},
\text{visual-only},\text{transactional},\text{brand},\text{date}\}$,
\begin{equation}
\hat{u}(\text{kind} \mid I) \;\propto\; \exp\!\Bigl(\textstyle\sum_{f \in \mathcal{F}} w_{\text{kind}, f}\, \phi_f(I)\Bigr),
\end{equation}
gives the score function used by Algorithm~\ref{alg:routing}, with the
weights $w_{\text{kind}, f}$ given by the constants listed there
($0.80, 0.90, 0.85, \ldots$). The per-domain \emph{learnings}
override (lines~$9$--$14$ of the algorithm) corresponds to the
forager's prior over patch quality from past visits, formally an
empirical Bayes update of the score with domain history as evidence.

\paragraph{Connection to dead-ends.} In the same framework, a
DeadEnd entry for URL $u$ with cause $c$ contributes a negative
utility prior $-\beta\, \mathbb{1}[\text{visit}(u)]$ on the worker's
next-step policy, formalising the ``do not re-enter the same patch''
prescription of foraging theory.

\end{document}